\documentclass{article}

\usepackage{microtype}
\usepackage{graphicx}
\usepackage{subfigure}
\usepackage{booktabs} %
\usepackage{nicefrac}
\usepackage{amsfonts}
\usepackage{amsmath}
\usepackage{amsthm}
\usepackage{amssymb}
\usepackage{dsfont}

\usepackage{xcolor}
\usepackage{color}
\usepackage{graphicx}

\usepackage[algo2e]{algorithm2e} 
\usepackage{verbatim}
\usepackage{xspace} %

\usepackage{enumerate}
\usepackage{enumitem}

\providecommand{\lin}[1]{\ensuremath{\left\langle #1 \right\rangle}}

\providecommand{\norm}[1]{\left\lVert#1\right\rVert}

  \providecommand{\R}{\mathbb{R}} %

  \DeclareMathOperator{\E}{{\mathbb E}}
  \providecommand{\EE}[2]{{\mathbb E}_{#1}\left.#2\right. }  %

  \DeclareMathOperator*{\argmin}{arg\,min}
  
  \DeclareMathOperator*{\supp}{supp}

  \providecommand{\0}{\mathbf{0}}
  \providecommand{\1}{\mathbf{1}}
  \renewcommand{\aa}{\mathbf{a}}
  \providecommand{\bb}{\mathbf{b}}

  \providecommand{\ee}{\mathbf{e}}

  \renewcommand{\gg}{\mathbf{g}}

  \providecommand{\xx}{\mathbf{x}}
  \providecommand{\yy}{\mathbf{y}}
  \providecommand{\zz}{\mathbf{z}}
  
  \providecommand{\mA}{\mathbf{A}}

  \providecommand{\mI}{\mathbf{I}}

  \providecommand{\mZ}{\mathbf{Z}}

  \providecommand{\cD}{\mathcal{D}}

  \providecommand{\cG}{\mathcal{G}}

  \providecommand{\cN}{\mathcal{N}}
  \providecommand{\cO}{\mathcal{O}}

  \providecommand{\cW}{\mathcal{W}}

  \usepackage{bm}
  \newcommand{\bxi}{\boldsymbol{\xi}}

  \usepackage[textwidth=5cm]{todonotes}
  
\providecommand{\mycomment}[3]{\todo[caption={},size=footnotesize,color=#1!20]{\textbf{#2: }#3}}%
\providecommand{\inlinecomment}[3]{%
  {\color{#1}#2: #3}}%
\newcommand\commenter[2]%
{%
  \expandafter\newcommand\csname i#1\endcsname[1]{\inlinecomment{#2}{#1}{##1}}
  \expandafter\newcommand\csname #1\endcsname[1]{\mycomment{#2}{#1}{##1}}
}

\newtheorem{proposition}{Proposition}
\newtheorem{lemma}{Lemma}

\newtheorem{definition}{Definition}
\newtheorem{remark}[lemma]{Remark}
\newtheorem{assumption}{Assumption}
\newtheorem{theorem}[lemma]{Theorem}

\usepackage[colorlinks=true,linkcolor=blue]{hyperref} 
\usepackage[capitalize,noabbrev]{cleveref}

\usepackage{url}

\usepackage{xcolor}

\commenter{seb}{red}
\commenter{nicolas}{green}
\commenter{ak}{magenta}
\commenter{mj}{red}

\graphicspath{{img/}}

\usepackage{hyperref}

\usepackage[accepted]{icml2020}

\icmltitlerunning{A Unified Theory of Decentralized SGD}

\begin{document}

\twocolumn[
\icmltitle{A Unified Theory of Decentralized SGD \\with Changing Topology and Local Updates}

\icmlsetsymbol{equal}{*}

\begin{icmlauthorlist}
\icmlauthor{Anastasia Koloskova}{equal,epfl}
\icmlauthor{Nicolas Loizou}{to}
\icmlauthor{Sadra Boreiri}{epfl}
\icmlauthor{Martin Jaggi}{epfl}
\icmlauthor{Sebastian U. Stich}{equal,epfl}
\end{icmlauthorlist}

\icmlaffiliation{epfl}{EPFL, Lausanne, Switzerland}
\icmlaffiliation{to}{Mila and DIRO, Universit\'{e} de Montr\'{e}al, Canada}

\icmlcorrespondingauthor{Anastasia Koloskova}{anastasia.koloskova@epfl.ch}
\icmlcorrespondingauthor{Sebastian U. Stich}{sebastian.stich@epfl.ch}

\icmlkeywords{Machine Learning, ICML}

\vskip 0.3in
]

\printAffiliationsAndNotice{\icmlEqualContribution} %

\newcommand{\compr}{\delta}%
\newcommand{\sgap}{\rho}%

\begin{abstract}
Decentralized stochastic optimization methods have gained a lot of attention recently, mainly because of their cheap per iteration cost, data locality, and their communication-efficiency.
In this paper we introduce a unified convergence analysis that covers a large variety of decentralized SGD methods which so far have required different intuitions, have different applications, and which have been developed separately in various communities. 
\\
Our algorithmic framework covers
local SGD updates and synchronous and pairwise gossip updates on adaptive network topology.
We derive universal convergence rates for smooth (convex and non-convex) problems and the rates interpolate between the heterogeneous (non-identically distributed data) and iid-data settings, recovering linear convergence rates in many special cases, for instance for over-parametrized models. %
Our proofs rely on weak assumptions (typically improving over prior work in several aspects) and recover (and improve) the best known complexity results for a host of important scenarios, such as for instance
coorperative SGD and 
federated averaging (local SGD).
\end{abstract}

\section{Introduction}
Training machine learning models in a non-centralized fashion can offer many advantages over traditional centralized approaches in core aspects such as data ownership, privacy, fault tolerance and scalability. 
In efforts to depart from the traditional parameter server paradigm \cite{dean2012large}, federated learning~\cite{konecny2016federated2,McMahan16:FedLearning,McMahan:2017fedAvg,Kairouz2019:federated} has emerged, but also fully decentralized approaches have been suggested recently---though yet still at a smaller scale than federated learning~\cite{Lian2017:decentralizedSGD,Assran:2018sdggradpush,KoloskovaLSJ19decentralized}.
However, the community has identified a host of challenges that come along with decentralized training:\ notably,  high communication cost~\cite{Tang2018:decentralized,Wang2019:matcha,Koloskova:2019choco}, a need for time-varying topologies~\cite{nedic2014distributed,Assran:2018sdggradpush} and data-heterogeneity~\cite{li2018fedprox,Kgoogle:cofefe,li2020feddane,Li2020:fedavg}.
It is imperative to have a good theoretical understanding of decentralized stochastic gradient descent (SGD) to predict the training performance of SGD in these scenarios and to assist the design of optimal decentralized training schemes for machine learning tasks.\looseness=-1 

In contrast to the centralized setting, where the convergence of SGD is well understood  \cite{Moulines2011:nonasymptotic,%
Rakhlin2012:bound_for_a_0,Dekel2012:minibatch}, the analyses of SGD in non-centralized settings are often application specific and  have  been historically developed separately in different communities, besides some recent efforts towards a unified theory. Notably, \citet{Wang2018:cooperativeSGD} propose a framework for decentralized optimization with non-heterogeneous data and \citet{Li2019:decentralized} study decentralized SGD for non-convex heterogeneous settings. We here propose a significantly extended framework that covers these previously proposed ones as special cases.

We provide tight convergence rates for a large family of decentralized SGD variants. Proving convergence rates in a unified framework is much more powerful than studying individual special cases on their own:\ We are not only able to recover many existing analyses and results, we can also often show improved rates  under more general setting. Remarkably, for instance for local SGD~\cite{zinkevich2010parallelized,Stich2018:LocalSGD,patel2019communication}
we show improved rates for the convex and strongly-convex case and recover the best known rates for the non-convex case under weaker assumptions than assumed in prior work (highlighted in Table~\ref{tab:localsgd}).

\subsection{Contributions}
\begin{itemize}%
 \item We present a unified framework for gossip based decentralized SGD methods that captures local updates and time-varying, randomly sampled, mixing distributions. Our framework covers a rich class of methods that previously needed individual  convergence analyses.
 \item Our theoretical results rely on weak assumptions that measure the strength of the noise and the dissimilarity of the functions between workers and a novel assumption on the expected mixing rate of the gossip algorithm. 
 This provides us with great flexibility on how to select the topology of the network and the mixing weights.
 \item We demonstrate the effectiveness and tightness of our results by exemplary showing that our framework gives the best convergence rates for local SGD for both, heterogeneous and iid.\ data settings, improving over all previous analyses on convex functions.
 \item We provide a lower bound that confirms that our convergence rates are tight on strongly convex functions.
 \item We empirically verify the tightness of our theoretical results on strongly convex functions and explain the impact of noise and data diversity on the convergence.
\end{itemize}

\section{Related Work}

The study of decentralized optimization algorithms can be tracked back at least to~\cite{Tsitsiklis1985:gossip}.
For the problem of computing aggregates (finding consensus) among clients, various gossip-based protocols have been proposed. For instance the push-sum algorithm~\cite{Kempe2003:gossip}, based on the intuition of mixing in Markov chains and allowing for asymmetric communication, or the symmetric randomized gossip protocol for averaging over arbirary graphs~\cite{Xiao2014:averaging,Boyd2006:randgossip} that we follow closely in this work.
For general optimization problems, the most common algorithms are either combinations of standard gradient based methods with gossip-type averaging step~\cite{Nedic2009:distributedsubgrad,Johansson2010:distributedsubgrad}, or 
specifically designed methods relying on problem structure, such as 
alternating direction method of multipliers (ADMM)~\cite{Wei2012:distributedadmm,Iutzeler2013:randomizedadmm}, dual averaging~\cite{Duchi2012:distributeddualaveragig,Nedic2015:dualavg,Rabbat2015:mirrordescent}, primal-dual methods~\cite{Alghunaim2019:pd}, or block-coordinate methods for generalized linear models~\cite{cola2018nips}.
There is a rich literature in the control community that discusses various special cases---motivated by particular applications---such as for instance asynchronity~\cite{Boyd2006:randgossip} or time-varying graphs~\cite{nedic2014distributed,Nedic2016:push}, see also \cite{Nedic2018:toplogy} for an overview. \looseness=-1

For the deterministic (non-stochastic) descentralized optimization a recent line of work developed optimal algorithms based on acceleration~\cite{Jakovetic2014:fast,Scaman2017:optimal,Scaman2018:non-smooth,Uribe:2018uk,Fallah2019:distributed}. 
In the machine learning context, decentralized implementations of stochastic gradient descent have gained a lot of attention recently~\cite{Lian2017:decentralizedSGD,Tang2018:d2,Assran:2018sdggradpush,KoloskovaLSJ19decentralized}, especially for the particular (but not fully decentralized) case of a star-shaped network topology, the federated learning setting~\cite{konecny2016federated2,McMahan16:FedLearning,McMahan:2017fedAvg,Kairouz2019:federated}.
Rates for the stochastic optimization are derived in~\cite{Shamir2014:distributedSO,Rabbat2015:mirrordescent}, under the assumption that the distributions %
on all nodes are equal. However, this is a very strong assumption for practical problems. %

It has been noted quite early that decentralized gradient based methods in heterogenous data setting
suffer from a `client-drift', i.e.\ the diversity in the functions on each node leads to a drift on each client towards the minima of $f_i$---potentially far away from the global minima of $f$. This phenomena has been discussed (and sometimes been adressed by modifing the SGD updates) for example in~\cite{shi2015extra,lee2015distributed,Nedic2016:drift}
and been rediscovered frequently in the context of stochastic optimization \cite{zhao2018federated,Kgoogle:cofefe}.
It is important to note that in analyses based on the bounded gradient assumption---which was traditionally assumend for analyzing SGD~\cite{Lacoste2012:simpler,Rakhlin2012:bound_for_a_0}---the diversity in the data distribution on each worker sometimes can be hidden in this generous upper bound and the analyses cannot distinguish between iid.\ and non-iid.\ data cases, such as e.g. in~\cite{Koloskova:2019choco,Nadiradze2019:PopSGD, Li2020:fedavg}.
In this work, we use much weaker assumptions and we show how the convergence rate depends on the similarity between the functions (by providing matching lower and upper bounds). Our results show that
in overparametrized settings no drift effects occur and linear convergence can be achieved similar as to the centralized setting~
\cite{Schmidt2013:fastconvergence,Needell2016:sgd,Ma2018:interpolation}.

For reducing communication cost, various techniques have been proposed. In this work we do not consider gradient compression techniques~\cite{Alistarh2017:qsgd,Stich2018:sparsifiedSGD,%
Tang2018:decentralized,Tang2019:squeeze,StichK19delays}---but such orthogonal techniques could be added on top of our scheme---and instead only focus on local updates steps which are often efficient in practice but challenging to handle in the theoretical analysis~\cite{McMahan:2017fedAvg,Stich2018:LocalSGD,yu2019parallel,LinSPJ2018local}.

\section{Setup}

We study the distributed stochastic optimization problem
\begin{align}
 f^\star := \min_{\xx \in \R^d} \bigg[\  f(\xx):=\frac{1}{n} \sum_{i=1}^n f_i(\xx) \ \bigg] \label{eq:f}
\end{align} 
where the components $f_i \colon \R^d \to \R$ are distributed among~$n$ nodes and are given in stochastic form:
\begin{align}\label{eq:f_i}
 f_i(\xx) := \E_{\xi_i \sim \cD_i} F_i(\xx,\xi_i),
\end{align}
where $\cD_i$ denotes the distribution of $\xi_i$ over parameter space~$\Omega_i$ on node $i$.
Standard empirical risk minimization is an important special case of this problem, when each~$\cD_i$ presents a finite number $m_i$ of elements $\{\xi_i^1,\dots,\xi_i^{m_i}\}$. Then $f_i$ can be rewritten as $f_i(\xx)= \frac{1}{m_i}\sum_{j=1}^{m_i} F_i(\xx,\xi_i^j)$. In the special case of $m_i=1$, for each $i \in [n]$, we further recover the deterministic distributed optimization problem.

It is important to note that we do not make any assumptions on the distributions $\cD_i$. This means that we especially cover hard heterogeneous machine learning problems where data is only available locally to each worker $i \in [n]:=\{1, \dots, n\}$ and the \emph{local minima} $\xx_i^\star := \argmin_{\xx \in \R^d} f_i(\xx)$, can be far away from the global minimizer of~\eqref{eq:f}. This covers a host of practically relevant problems over decentralized training data, as in federated learning (motivated by privacy), or large datasets stored across datacenters or devices (motivated by scalability).
We will discuss several important examples in Section~\ref{sec:noise} below.

\subsection{Assumptions on the objective function $f$}
For all our theoretical results we assume that $f$ is smooth.

\let\origtheassumption\theassumption

\edef\oldassumption{\the\numexpr\value{assumption}+1}

\setcounter{assumption}{0}
\renewcommand{\theassumption}{\oldassumption\alph{assumption}}

\begin{assumption}[$L$-smoothness]\label{a:lsmooth}
Each function $F_i(\xx, \xi)\colon \R^d \times \Omega_i \to \R$, $i \in [n]$
is differentiable for each $\xi \in \supp(\cD_i)$ and there exists a constant $L \geq 0$ such that for each $\xx, \yy \in \R^d, \xi \in \supp(\cD_i)$:
\begin{align}
&\norm{\nabla F_i(\yy, \xi) - \nabla F_i(\xx,\xi) } \leq L \norm{\xx -\yy}\,. \label{eq:F-smooth}%
\end{align}
\end{assumption}
Sometimes it will be enough to just assume smoothness of $f_i$ instead.

\begin{assumption}[$L$-smoothness]\label{a:lsmooth_nc}
	Each function $f_i(\xx) \colon \R^d \to \R$, $i \in [n]$
	is differentiable and there exists a constant $L \geq 0$ such that for each $\xx, \yy \in \R^d$:
	\begin{align}\label{eq:smooth_nc}
	&\norm{\nabla f_i(\yy) - \nabla f_i(\xx) } \leq L \norm{\xx -\yy}\,. %
	\end{align}
\end{assumption}
\let\theassumption\origtheassumption
\setcounter{assumption}{1}
\begin{remark}
 Clearly, Assumption~\ref{a:lsmooth_nc} is more general than Assumption~\ref{a:lsmooth}. Moreover, for convex $F(\yy,\xi)$ Assumption~\ref{a:lsmooth} implies Assumption~\ref{a:lsmooth_nc}~\cite{Nesterov2004:book}.
\end{remark}

Assumption~\ref{a:lsmooth_nc} is quite common in the literature \citep[e.g.][]{Lian2017:decentralizedSGD,Wang2018:cooperativeSGD} but sometimes also the stronger Assumption~\ref{a:lsmooth} is assumed~\cite{Nguyen2018:async}. We here use this  version in the convex case only, to allow for a more general assumption on the noise instead (see Section~\ref{sec:noise} below).

For some of the derived results we need in addition convexity. Specifically, $\mu$-convexity for a parameter $\mu \geq 0$.

\begin{assumption}[$\mu$-convexity]
\label{a:strong}
Each function $f_i \colon \R^d \to \R$, $i \in [n]$ is $\mu$-(strongly) convex for constant $\mu \geq 0$. That is, for all $\xx,\yy \in \R^d$:
\begin{align}
 f_i(\xx)-f_i(\yy) + \frac{\mu}{2}\norm{\xx-\yy}^2_2 \leq \lin{\nabla f_i(\xx),\xx-\yy}\,. \label{eq:strongconv}
\end{align}
\end{assumption}

\subsection{Assumptions on the noise}\label{sec:noise}
We now formulate our conditions on the noise. 
For the convergence analysis of SGD on smooth convex functions it is typically enough to assume a bound on the noise at the optimum only~\cite{Needell2016:sgd,Bottou2018:book,gower2019sgd,Stich19sgd}. Similarly, to express the diversity of the functions $f_i$ in the convex case it is sufficient to measure it only at the optimal point $\xx^\star$ (such a point always exists for strongly convex functions).

\let\origtheassumption\theassumption

\edef\oldassumption{\the\numexpr\value{assumption}+1}

\setcounter{assumption}{0}
\renewcommand{\theassumption}{\oldassumption\alph{assumption}}

\begin{assumption}[Bounded noise at the optimum]\label{a:opt}
Let  $\xx^\star = \argmin f(\xx)$ and define
\begin{align}
\zeta_i^2 &:= \norm{\nabla f_i(\xx^\star)}^2_2\,, & & \textstyle \bar \zeta^2 := \frac{1}{n}\sum_{i=1}^n \zeta_i^2\,. \label{eq:grad_opt}
\end{align}
Further, define
	\begin{align}
	\sigma_i^2 := \EE{\xi_i}{\norm{\nabla F_i(\xx^\star, \xi_i) - \nabla f_i(\xx^\star)}}^2_2  \,, \label{eq:noise_opt}
	\end{align}	
	and similarly as above, $\bar \sigma^2 := \frac{1}{n}\sum_{i=1}^n \sigma_i^2$. We assume that $\bar \sigma^2$ and $\bar \zeta^2$ are bounded.
\end{assumption}
Here, $\bar \sigma^2$ measures the noise level, and $\bar \zeta^2$ the diversity of the functions $f_i$. If all functions are identical, $f_i=f_j$, for all $i, j$, then $\bar \zeta^2 = 0$. Many prior work in the context of stochastic decentralized optimization often assumed bounded diversity and bounded noise \emph{everywhere}  \citep[such as e.g.][]{Lian2017:decentralizedSGD,Tang2018:d2}, whereas we here only need to assume this bound locally at $\xx^\star$.

For the non-convex case---where a unique $\xx^\star$ does not necessarily exist---we generalize Assumption~\ref{a:opt} to:
\begin{assumption}[Bounded noise]\label{a:opt_nc}
	We assume that there exists constants $P$, $\hat \zeta$ such that $\forall \xx \in \R^d$,
	\begin{align} \textstyle
	\frac{1}{n} \sum_{i = 1}^n \norm{\nabla f_i(\xx)}_2^2 \leq \hat \zeta^2 + P \norm{\nabla f(\xx)}^2_2 \,, \label{eq:grad_opt_nc}
	\end{align}
	and constants  $M$, $ \hat \sigma $ such that $\forall \xx_1, \dots \xx_n \in \R^d$
	\begin{align} \textstyle
	\Psi \leq \hat \sigma^2 +  \frac{M}{n} \sum_{i = 1}^n \norm{\nabla f_i(\xx_i)}^2_2 \,, \label{eq:noise_opt_nc}
	\end{align}
	where $\Psi := \frac{1}{n} \sum_{i = 1}^n \EE{\xi_i}{\norm{\nabla F_i(\xx_i, \xi_i) - \nabla f_i(\xx_i)}}^2_2$.
\end{assumption}
\let\theassumption\origtheassumption
\setcounter{assumption}{3}
We see that Assumption~\ref{a:opt} is weaker than Assumption~\ref{a:opt_nc} as it only needs ho hold for $\xx_i = \xx^\star$. Further, it is important to note that we do not assume a uniform bound on the variance \citep[as many prior work, such as][]{Li2019:decentralized,Tang2018:d2,Lian2017:decentralizedSGD,Assran:2018sdggradpush} but instead allow the bound on the noise and the diversity to grow with the gradient norm (similar assumptions are common in the convex setting~\cite{Bottou2018:book}).

\textbf{Discussion.} We now show that the Assumption~\ref{a:opt_nc} is weaker than
assuming a uniform upper bound on the noise. %
The uniform variance bound is given as
\begin{align*}
\E{\norm{\nabla F_i(\xx, \xi_i) - \nabla f_i(\xx)}}_2^2 &\leq \sigma_{\rm unif}^2\,, & &\forall\xx \in \R^d\,,
\end{align*}
similarly for the similarity of functions between nodes
\begin{align*}\textstyle
\frac{1}{n}\sum_{i = 1}^n {\norm{\nabla f_i(\xx) - \nabla f(\xx)}}_2^2 &\leq \bar \zeta_{\rm unif}^2\,, & & \forall \xx \in \R^d\,.
\end{align*}
By recalling the equality $\frac{1}{n}\sum_{i = 1}^n {\norm{\aa_i  - \bar \aa }}_2^2 = \frac{1}{n}\sum_{i = 1}^n \norm{\aa_i}_2^2 - \norm{\bar \aa}_2^2$ for $\aa_i \in \R^d, \bar \aa = \frac{1}{n}\sum_{i = 1}^n \aa_i$, it is easy to check that these two bounds imply Assumption~\ref{a:opt_nc} with $P = 1$, $M = 0$, $\hat \sigma^2 = \sigma_{\rm unif}^2$ and $\hat \zeta^2 = \bar \zeta_{\rm unif}^2$. Thus, our assumptions are weaker and $\hat \zeta^2$ and $\hat \sigma^2$ can be much smaller than $\bar \zeta_{\rm unif}^2, \sigma_{\rm unif}^2$ in general.

A second common assumption is to assume that the (stochastic) gradients are uniformly bounded  \citep[e.g.][]{Koloskova:2019choco,Li2020:fedavg}, that is
\begin{align*}
\E{\norm{\nabla F_i(\xx, \xi_i)}}_2^2 \leq G^2\,,
\end{align*}
for a constant $G$. Under the bounded gradient assumption, Assumption~\ref{a:opt_nc} is clearly satisfied, as all terms on the left hand side of~\eqref{eq:grad_opt_nc} and~\eqref{eq:noise_opt_nc} can be upper bounded by $2G^2$.

\subsection{Notation}

We use the notation $\xx_i^{(t)}$ to denote the iterates on node $i$ at time step $t$. We further define the average
\begin{align}
 \bar \xx^{(t)} := \textstyle \frac{1}{n}\sum_{i=1}^n \xx_i^{(t)}\,.
\end{align}

We use both vector and matrix notation whenever it is more convenient, and define
\begin{align}
X^{(t)} := \left[ \xx_1^{(t)},\dots, \xx_n^{(t)}\right] \in \R^{d\times n}\label{eq:notation_sgd}
\end{align}
and likewise define $\bar X^{(t)} := \left[ \bar \xx^{(t)},\dots, \bar \xx^{(t)}\right] \equiv X^{(t)}\frac{1}{n} \1\1^\top$.

\section{Decentralized (Gossip) SGD}
We now present the generalized decentralized SGD framework.
Similar to existing works \cite{Lian2017:decentralizedSGD,Wang2018:cooperativeSGD, Li2019:decentralized} our proposed method allows only decentralized communications. That is, the exchange of information (through \emph{gossip} averaging) can only occur between connected
nodes (neighbors). 
The algorithm (outlined in Algorithm~\ref{alg:random_W_decentr_sgd_matrix}) consists of two phases: 
(i) stochastic gradient updates, performed locally on each worker (lines 4--5), followed by a 
(ii) consensus operation, where nodes average their values with their neighbors (line 6).

The gossip averaging protocol can be compactly written in matrix notation, with $\cN_i^{(t)}:=\{j \colon w_{ij}^{(t)}>0\}$ denoting the neighbors of node $i$ at iteration $t$:
\begin{align*}
 X^{(t+1)} &= X^{(t)}W^{(t)}  & &\!\! \Leftrightarrow \!\! &  &  & \xx_i^{(t+1)} &= \textstyle \sum_{j \in \cN_i^{(t)}} w_{ij}^{(t)} \xx_j^{(t)}\,,
\end{align*}
where
the mixing matrix $W^{(t)} \in [0,1]^{n \times n}$  encodes the network structure at time $t$ and the averaging weights (nodes $i$ and $j$ are connected if $w^{(t)}_{ij} > 0$).

Our scheme shows great flexibility as the mixing matrices can change over iterations and moreover can be selected from (changing) distributions. %

\begin{definition}[Mixing matrix]\label{def:valid_mixing}
A symmetric  ($W\!=\!W^\top$) doubly stochastic ($W\1 \!=\! \1$, $\1^\top W\! = \!\1^\top\!$) matrix $W \!\in\! [0,1]^{n \times n}$.
\end{definition}

\subsection{Algorithm}

\begin{algorithm}[H]
	\caption{\textsc{decentralized SGD}}\label{alg:random_W_decentr_sgd_matrix}
	\resizebox{\linewidth}{!}{
	\begin{minipage}{1.05\linewidth}
	\begin{algorithmic}[1]
		\INPUT{for each node $i\in [n]$ initialize $\xx_i^{(0)} \in \R^d$, 
		\\ stepsizes $\{\eta_t\}_{t=0}^{T-1}$, number of iterations $T$, 
		\\ mixing matrix distributions $\cW^{(t)}$ for $t \in [0, T]$}

		\FOR{$t$\textbf{ in} $0\dots T$}
			\STATE Sample $W^{(t)} \sim \cW^{(t)}$ 
		\STATE {\it In parallel (task for worker $i, i \in [n]$)}
		\STATE Sample $\xi_i^{(t)}$, compute $\gg_i^{(t)} := \nabla F_i(\xx_i^{(t)}\!, \xi_i^{(t)})$
		\STATE $\xx_i^{(t + \frac{1}{2})} = \xx_i^{(t)} - \eta_t \gg_i^{(t)}$  \hfill $\triangleright$ stochastic gradient updates
		\STATE $\xx_i^{(t + 1)} := \sum_{j \in \cN_i^{t}} w_{ij}^{(t)} \xx_j^{(t + \frac{1}{2})}$  \hfill  $\triangleright$ gossip averaging
		\ENDFOR
	\end{algorithmic}
	\end{minipage}
	}
\end{algorithm}
\vspace{-2mm}
In each iteration in Algorithm~\ref{alg:random_W_decentr_sgd_matrix} a new mixing matrix $W^{(t)}$ is sampled from a possibly time-varying distribution $\cW^{(t)}$,
 $t \in \{0, \dots, T\}$ (we will show below that also degenerate mixing matrices, for instance $W^{(t)}=\mI_n$ which implies \emph{no communication} in round $t$, are possible choices). %
We will discuss several important instances below, but first we now state our assumption on the quality of the mixing matrices. 
This assumption is novel in the literature to the best of our knowledge and a natural generalization of earlier versions.

\subsection{New assumption on mixing matrices}
\label{sec:p}

We recall that for randomized gossip averaging with a randomly sampled mixing matrix $W \sim \cW$ it holds
\begin{align}
\E_{W }\norm{X W - \bar X}_F^2 \leq (1 - p) \norm{X-\bar X}_F^2 \,,\label{eq:boyd}
\end{align}
for a value $p\geq 0$ (related to the spectrum of $\E W^\top W$), 
that is, the averaging step brings the values in the columns of $X \in \R^{d \times n}$ closer to their row-wise average $\bar X  :=  X \cdot \frac{1}{n}\1\1^\top$ in expectation~\citep[see e.g.][]{Boyd2006:randgossip}. %

In our analysis it will be enough to \emph{assume} that a property similar to~\eqref{eq:boyd} holds for the \emph{composition} of mixing matrixes, and does not necessarily hold for every single step. %

\begin{assumption}[Expected Consensus Rate]\label{a:avg_distrib}
We assume that there exists two constants $p \in (0,1]$ and integer $\tau \geq 1$ such that for all matrices $X\in \R^{d \times n}$ and all integers $\ell \in \{0,\dots,T / \tau\}$,
\begin{align}
	\EE{W}{\norm{X W_{\ell, \tau} - \bar X}_F^2} &\leq (1 - p) \norm{X-\bar X}_F^2 \,, \label{eq:p}
	\end{align}
	where $W_{\ell,\tau} = W^{((\ell + 1)\tau - 1)}\cdots W^{(\ell \tau)}$ and $\bar{X} := X \frac{\1\1^\top}{n}$ and $\mathbb E$ is taken over the distributions $W^{(t)} \sim \cW^{(t)}$ and indices $ t \in \{\ell \tau , \dots, (\ell + 1)\tau - 1\}$. 
\end{assumption}

It is crucial to observe that this assumption does not require every %
realization
$W$ to satisfy a decrease property as for the standard analysis, it is enough if it holds over the concatenation of $\tau$ mixing steps.
This assumption differs from the connectivity assumptions sometimes used in the control community. For example \citet{nedic2014distributed} require strong connectivity of the graph after every $\tau$ steps, whereas we here do not require this (for example, even sampling one single random edge leads to a positive decrease in expectation, whereas to ensure connectivity one would need to perform $\Omega(n)$ pairwise communications). This means that our bounds are typically much tighter that bounds derived on the strong connectivity assumption. However, as we require $W$ to be symmetric, our setting is less general than the one considered in~\cite{nedic2017achieving,%
xi2017dextra,saadatniaki2018optimization,%
Assran2018:async,scutari2019distributed,Assran:2018sdggradpush}. \looseness=-1

Commonly used weights are for instance the Metropolis-Hastings weights $w_{ij}=w_{ji} = \min\bigl\{\frac{1}{\deg(i)+1},\frac{1}{\deg(j)+1}\bigr\}$ for $(i,j) \in E$, see also \cite{Xiao2014:averaging,Boyd2006:randgossip} for further guidelines. With these weights, the values of $p$ for commonly used graphs are $p=1$ for the complete graph, $p=\Theta\bigl(\frac{1}{n}\bigr)$ for 2-$d$ torus on $n$ nodes, and $p=\Theta\bigl(\frac{1}{n^2}\bigr)$ for a cycle on $n$ nodes. Intuitively, $p^{-1/2}$ correlates with the diameter of the graph and is related to the mixing time of Markov chains. A commonly studied randomized scheme is the pairwise random gossip algorithm~\cite{Boyd2006:randgossip,loizou2019revisiting}, where one edge at a  time is sampled  from an underlying graph $\cG = ([n], E)$, i.e.\  the a random mixing matrix $\mZ_{i,j} := \mI_n- \frac{1}{2}(\ee_i-\ee_j)(\ee_i-\ee_j)^\top$, for all edges in the graph $(i,j) \in E$, where $\ee_i \in \R^n$ is the $i^{th}$ coordinate vector. In this case  $p=\rho(\cG)/|E|$, where $\rho(\cG)$ denotes the algebraic connectivity of the network \cite{fiedler1973algebraic,Boyd2006:randgossip,LoizouRichtarik}. For example, with the complete graph as base graph, pairwise gossip attains $p=\Theta\bigl(\frac{1}{n^2}\bigr)$, i.e.\ enjoys equally fast mixing as averaging over a (fixed) cycle (which requires $n$ pairwise communications per round). 

\section{Examples Covered in the Framework}
\label{sec:examples}
Our framework is very general and covers many special cases previously introduced in the literature. %

\newcommand{\highlight}{\bfseries$\bullet$~}

\subsection{Fixed Sampling Distribution ($\tau = 1$, $\cW^{(t)} \equiv \cW$)}
The simplest instances of Algorithm~\ref{alg:random_W_decentr_sgd_matrix} arise when the mixing matrix $W$  is kept constant over the iterations. By choosing the fully connected matrix $W = \frac{1}{n}\1\1^\top$ we recover {\highlight centralized mini-batch SGD} \cite{Dekel2012:minibatch} and by choosing an arbitrary connected $W$, we recover {\highlight decentralized SGD} \cite{Lian2017:decentralizedSGD}.

To reduce communication overheads, it has been proposed to choose sparse (not necessarily connected) subgraphs of the network topology. For instance in {\highlight MATCHA} \cite{Wang2019:matcha} it is proposed to sample edges from a matching decomposition of  the underlying network topology, therefore allowing for pairwise communications between nodes. Whilst no explicit values of $p$ were given for this approach,
 for the simpler instance of {\highlight pairwise randomized gossip}~\cite{Boyd2006:randgossip,ram2010asynchronous,%
lee2015asynchronous,loizou2019revisiting} we have $p=\Theta\bigl(\frac{1}{n}\bigr)$, thus by sampling a linear number of (independent) edges---not necessarily a matching---we approximately have $p=\tilde\Theta\big(1\bigr)$ for this {\highlight repeated pairwise randomized gossip} variant.
This approach can be generalized to {\highlight randomized subgraph gossip}, where a subgraph of the base topology is selected for averaging. A special case of this is {\highlight clique gossip} \cite{liu2019clique},
or an alternative variant is
to {\highlight sample from a fixed set of communication topologies} (known to all decentralized) workers.

One noteably instance of this type is {\highlight loopless local decentralized SGD} where the mixing matrix is (a fixed) $W$ with probability  $\frac{1}{\tau}$, and $\mI_n$ with probability $1-\frac{1}{\tau}$, for a parameter $\tau \geq 1$. This algorithm mimicks the behavior of the local SGD (see subsection below), commonly analyzed for $W=\frac{1}{n}\1 \1^\top$ only, but the loopless variant is much easier to analyze (with $p$ decreased by a factor of $\tau$, but no need to consider local steps explicitly in the analysis.).

\subsection{Periodic Sampling ($\tau > 1$, $\cW^{(t)} \equiv \cW^{(t+\tau)} $)}
Our analysis covers the empirical (finite-sample) versions of the aforementioned algorithms, for instance  {\highlight alternating decentralized SGD} that sweeps through $\tau$ fixed mixing matrices. A special algorithm of this type is {\highlight local SGD}~\cite{Coppola2015:IPM,Zhou2018:local,Stich2018:LocalSGD} where averaging on the complete graph is performed every $\tau$ iterations and only local steps are performed otherwise (mixing matrix $\mI_n$ for $\tau-1$ steps). 

Our analysis covers also natural extensions such as {\highlight decentralized local SGD} where mixing is performed with an arbitrary matrix $W$, and {\highlight random decentralized local SGD} where the mixing matrix is sampled from a distribution. More generally, our framework also  allows to combine local steps with all of the examples described in the previous section.

\subsection{Non-Periodic Sampling}
It is not necessary to have a periodic structure, 
it is sufficient that the composition of every $\tau$ consecutive mixing matrixes satisfies Assumption~\ref{a:avg_distrib}. For instance as in {\highlight distriributed SGD over time-varying graphs} \cite{nedic2014distributed}.

\subsection{Other Frameworks}
\label{sec:comparisontoframework}
In contrast to many prior works, we here allow the topology and the averaging weights to change between iterations.
Our framework covers
{\highlight Cooperative SGD}~\cite{Wang2018:cooperativeSGD} which considers only the IID data case ($f_i = f_j$) with local updates and a fixed mixing matrix $W$, and the recently proposed {\highlight periodic decentralized SGD} \citep{Li2019:decentralized} that allows for multiple local update and multiple mixing steps (for fixed $W$) in a periodic manner. None of these work considered sampling of the mixing matrix and do only provide rates for non-convex functions.

\section{Convergence Result}
In this section we present the convergence results for decentralized SGD variants that fit the template of Algorithm~\ref{alg:random_W_decentr_sgd_matrix}.

\subsection{Complexity Estimates (Upper Bounds)}
\begin{theorem}\label{thm:summary}
For schemes as in Algorithm~\ref{alg:random_W_decentr_sgd_matrix} with mixing matrices such as in Assumption~\ref{a:avg_distrib} and any target accuracy $\epsilon > 0$ there exists a (constant) stepsize (potentially depending on $\epsilon$) such that the accuracy can be reached after at most the following number of iterations $T$:\\
\textbf{Non-Convex:} Under Assumption~\ref{a:lsmooth_nc} and~\ref{a:opt_nc}, it holds $\frac{1}{T + 1}\sum_{t = 0}^T \E \norm{\nabla f(\bar{\xx}^{(t)})}_2^2 \leq \epsilon$ after \par
\resizebox{\linewidth}{!}{
\vbox{
\begin{align*}
 \cO \left( \frac{ \hat \sigma^2}{n \epsilon^2}  +  \frac{ \hat \zeta  \tau \sqrt{M+1} + \hat \sigma \sqrt{p\tau}   }{ p \epsilon^{3/2}}   + \frac{\tau  \sqrt{(P + 1)(M+1)}}{p \epsilon}  \right) \cdot L F_0
\end{align*}
}}
iterations. If we in addition assume convexity,\\
\textbf{Convex:} Under Assumption~\ref{a:lsmooth}, \ref{a:opt} and~\ref{a:strong} for $\mu \geq 0$, the error  $ \frac{1}{(T+1)} \sum_{t=0}^T  (\E{f(\bar\xx^{(t)})} - f^\star) \leq \epsilon$ after
\begin{align*}
   \cO \left( \frac{\bar \sigma^2 }{n \epsilon^2} + \frac{\sqrt{L}( \bar \zeta  \tau + \bar \sigma \sqrt{p\tau} )  }{ p \epsilon^{3/2}}   + \frac{L \tau}{p \epsilon} \right) \cdot R_0^2
\end{align*}
iterations, and  if $\mu > 0$, \\
\textbf{Strongly-Convex:}
then $ \sum_{t= 0}^T \frac{w_t}{W_T} (\E{f(\bar\xx^{(t)})} - f^\star) + \mu \E {\| \bar{\xx}^{(T+1)} - \xx^\star} \|^2 \leq \epsilon$ for\footnote{$\tilde{\cO}/\tilde{\Omega}$-notation hides constants and polylogarithmic factors.}
 \begin{align*}
  \tilde \cO \left( \frac{\bar \sigma^2}{ \mu  n \epsilon} + \frac{\sqrt{L}( \bar \zeta  \tau + \bar \sigma \sqrt{p\tau} )  }{\mu p  \sqrt{\epsilon}}   + \frac{L  \tau }{\mu p} \log \frac{1}{\epsilon} \right)
 \end{align*}
 iterations, for positive weights $w_t$ and $F_0 := f(\xx_0)- f^\star$ and $R_0=\norm{\xx_0 -\xx^\star}$ denote the initial errors.
\end{theorem}

\begin{table*}
\begin{minipage}{\textwidth}
\caption{Comparison of convergence rates for Local SGD in non-iid settings, most recent results. We improve over the convex results, and recover the non-convex rate of \citet{Li2019:decentralized}.}
\label{tab:localsgd} 
\resizebox{\linewidth}{!}{
\begin{tabular}{llllll}\toprule[1pt]
Reference & \multicolumn{5}{c}{convergence to $\epsilon$-accuracy} \\  \cmidrule{2-2} \cmidrule{4-4} \cmidrule{6-6}
          & strongly convex & \phantom{-} & convex & \phantom{-} & non-convex \\ \midrule
\citet{Li2020:fedavg} & $\cO \left( \frac{\bar \sigma^2}{n \mu^2 \epsilon} + \frac{ \tau^2 \bar \zeta^2}{\mu^2 \epsilon} \right)$\footnote{The paper relies on slightly different assumptions (bounded gradients and different measure of dissimilarity). For better comparison of the rates we write here $\bar \zeta^2$ instead (which is strictly smaller than their parameter).}  & & - & & -\\
\citet{Khaled2020:local} &  - & & $\cO \left( \frac{\bar \sigma^2 + \bar \zeta^2}{n \epsilon^2} + \frac{\sqrt{L}\tau (\bar \zeta + \bar \sigma )}{\epsilon^{3/2}}  + \frac{L \tau }{\epsilon}  \right)$ & & -\\
\citet{Li2019:decentralized} & - & & - & &  $\cO \left(\frac{L \bar \sigma^2}{n \epsilon^2} + \frac{L(\tau \bar \zeta + \sqrt{\tau}\bar \sigma  )}{\epsilon^{3/2}}  + \frac{L\tau}{\epsilon}  \right) $ \\
	this work & $\tilde \cO \left(\frac{\bar \sigma^2 }{n \mu \epsilon} + \frac{\sqrt{L } (  \tau \bar \zeta + \sqrt{\tau} \bar \sigma)}{\mu \sqrt{\epsilon}} + \kappa \tau \right)$ & & $ \cO \left(\frac{\bar \sigma^2}{n \epsilon^2} + \frac{\sqrt{L} (\tau \bar \zeta + \sqrt{\tau} \bar \sigma ) }{ \epsilon^{3/2}} + \frac{L\tau}{\epsilon} \right)$ & & $\cO \left(\frac{L \hat \sigma^2 }{n \epsilon^2} + \frac{L ( \tau \hat \zeta + \sqrt{\tau} \hat \sigma)  }{\epsilon^{3/2}}  + \frac{L\tau}{\epsilon} \right) $
\\ \bottomrule[1pt]
\end{tabular}}
\end{minipage}
\end{table*}

\subsection{Lower Bound}
We now show that the terms depending on $\bar \zeta$ are necessary for the strongly convex setting and cannot be removed by an improved analysis. %
\begin{theorem}\label{thm:lower_bound}
	For $n>1$ there exists strongly convex and smooth functions $f_i \colon \R^d \to \R$, $i \in [n]$ with $L = \mu = 1$ and without stochastic noise ($\bar \sigma^2 = 0$), such that Algorithm~\ref{alg:random_W_decentr_sgd_matrix} for every constant mixing matrix $W^{(t)}\equiv W$ with $p < 1$ (see Assumption~\ref{a:avg_distrib}) for $\tau = 1$, requires
\begin{align*}
T = \tilde \Omega \left( \dfrac{\bar\zeta  (1 - p)}{\sqrt{\epsilon} p }\right)
\end{align*}
iterations to converge to accuracy $\epsilon$.
\end{theorem}

\subsection{Discussion}
Exemplary, we focus in our discussion %
on the strongly convex case only. For strongly convex functions we prove that the expected function value suboptimality decreases as
	\begin{align*}
	\tilde\cO \left(\frac{\bar\sigma^2}{n \mu T} + \frac{L (\tau^2 \bar\zeta^2 + \tau p \bar\sigma^2)}{\mu^2 p^2 T^2} + \frac{L\tau R_0^2}{p} \exp\left[- \frac{\mu T p}{\tau L}\right] \right)\,
	\end{align*}
where $T$ denotes the iteration counter.
We now argue that this rate is optimal up to acceleration.

\textbf{Stochastic Terms.}
If $\bar\sigma^2>0$ the convergence rate is asymptotically dominated by the first term, which cannot be further improved for stochastic methods~\cite{nemirovskyyudin1983}.
We observe that the dominating first term indicates a linear speedup in the number of workers $n$, and no dependence on the number of local steps $\tau$, the mixing parameter $p$ or the dissimilarity parameter $\bar \zeta^2$. 
This means that decentralized SGD methods are ideal for the optimization in the high-noise regime even when network connectivity is low and number of local steps is large (see also \cite{Chaturapruek2015:noise} and recent work~\cite{Pu2019:independent}). 
In our rates the variance $\bar \sigma^2$ parameter also appears in the second term, but affects the convergence only mildly (for $T= \Omega(\tau n/p)$ this second term gets dominated by the first one).

\textbf{Optimization Terms.} 
Even when $\bar \sigma^2=0$, the convergence of decentralized SGD only  \emph{sublinear} when $\bar\zeta^2 > 0$:\footnote{Except for the special case when $p=1$ (fully connected graph, such as for mini-batch SGD). In this case the rate does not depend on $\bar \zeta^2$. We detail this (known result) in the appendix.} 
\begin{align*}
 \tilde \cO \left(  \frac{L \tau^2 \bar\zeta^2}{\mu^2 p^2 T^2} +  \frac{L\tau R_0^2}{p} \exp\left[- \frac{\mu T p}{\tau L}\right] \right)\,.
\end{align*}
The dependence on the dissimilarity $\bar \zeta^2$ cannot be removed in general as we show in Theorem~\ref{thm:lower_bound}.
These results show that decentralized SGD methods
 without additional modifications \citep[see also][]{shi2015extra,Kgoogle:cofefe} cannot converge linearly.
 
We can further observe see that the rates only depend on the ratio $p/\tau$, but not on $p$ or $\tau$ individually. This also means that the rates for local variants of decentralized SGD are the same as for their loopless variants (when the mixing is performed with probability $\frac{1}{\tau}$ only).
The error term depending on $R_0^2$ vanishes exponentially fast, as expected for SGD methods~\cite{Moulines2011:nonasymptotic}. The linear dependence on $ \frac{L }{\mu p }$ (the therm in the exponent) is expected here, as we use non-accelerated first order schemes and standard gossip. This term could potentially  be improved to $\smash{\bigl(\frac{L}{\mu p}\bigr)^{1/2}}$ with acceleration techniques, such as in~\cite{Scaman2017:optimal}. The linear dependence on $\tau$ cannot further be improved in general. This follows from the lower bound for the communication complexity of distributed convex optimization~\cite{Arjevani2015:lowerbound}, as the number of communication rounds is at most $\frac{T}{\tau}$ (no communication happens during the local steps). However, when $\bar \zeta^2 = 0$ (as for instance the case for identical functions $f_i$ on each worker), this lower bound becomes vacuous and improvement of the dependence on $\tau$ might be possible (which we cannot not exploit here).

\paragraph{Linear Convergence for Overparametrized Settings.}
In overparametrized problems, there exists always $\xx^\star$ s.t.\ $\norm{\nabla f_i(\xx^\star)}^2=0$, that is $\bar \sigma^2 =0$ and $\bar \zeta^2=0$. We prove here that decentralized SGD converges linearly in this case, similarly to mini-batch SGD~\cite{Moulines2011:nonasymptotic,%
Schmidt2013:fastconvergence,Needell2016:sgd,%
Ma2018:interpolation,gower2019sgd,loizou2020stochastic}.

\begin{figure*}[t]
	\centering
	\begin{minipage}{\textwidth}
		\centering
		\hfill
		\includegraphics[width=0.8\linewidth]{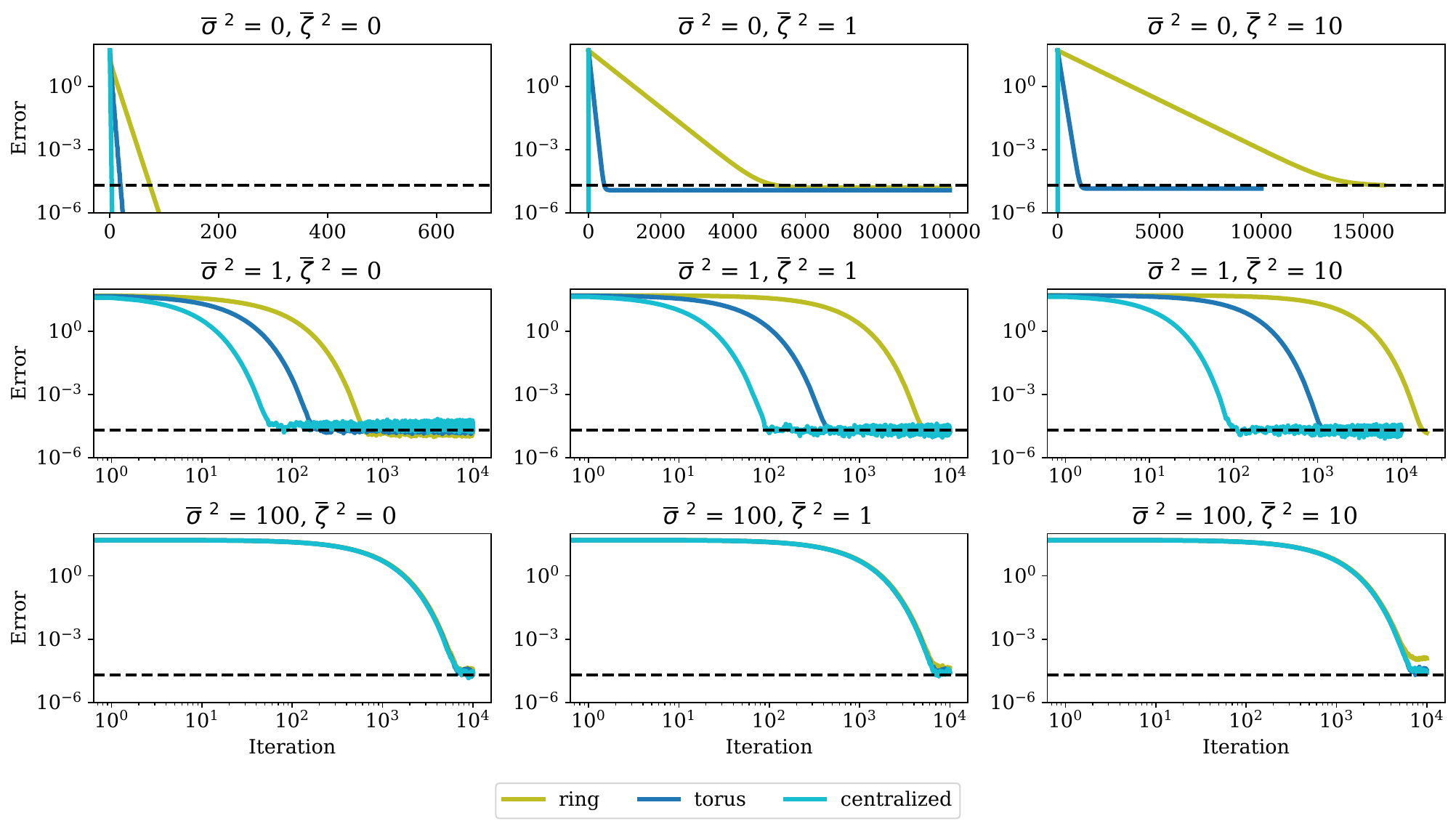}
		\hfill \null
		\caption{%
        Convergence of $\frac{1}{n} \sum_{i = 1}^n \bigl \| \xx_i^{(t)} - \xx^\star \bigr\|^2_2$ to target accuracy $\epsilon = 10^{-5}$ for different problem difficulty ($\bar \sigma^2$ increasing to the bottom, $\bar \zeta^2$ increasing to the right), and different topologies on $n=25$ nodes, $d = 50$. Stepsizes were tuned for each experiment individually to reach target accuracy in as few iterations as possible.
		}%
		\label{fig:exp}
	\end{minipage}%
\end{figure*}

\begin{figure}[b!]
	\centering
	\hfill
	\includegraphics[width=0.8\linewidth]{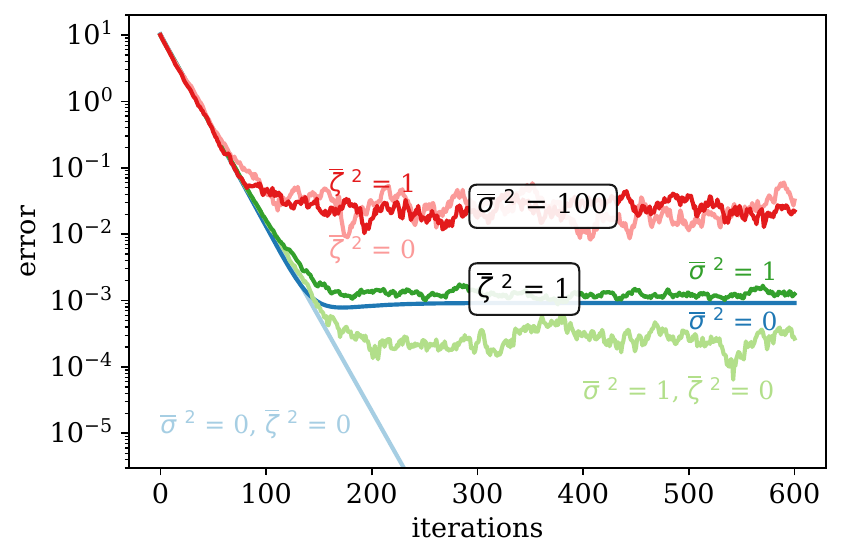}
	\hfill \null
	\caption{%
    Problem setup. Parameters $\bar \sigma^2$ and $\bar \zeta^2$ change the noise level and the difficulty of the problem. (Here we depict $\frac{1}{n} \sum_{i = 1}^n \bigl \| \xx_i^{(t)} - \xx^\star \bigr\|^2_2$ on the ring with $n = 25$ nodes, $d = 10$, using fixed stepsize $\eta = 10^{-2}$ for illustration. }%
	\label{fig:stepsize}
\end{figure}

\section{Special Cases: Highlights}
Our rates apply to all the examples discussed in Section~\ref{sec:examples} and of course we could design even more variants and combinations of these schemes. This gives great flexibility in designing new schemes and algorithms for future applications. We leave the exploration of the trade-offs in these approaches for future work, and highlight here only a few special cases that could be of particular interest. %

\subsection{Best Rates for Local SGD}
Local SGD is a simplified version of the federated averaging algorithm~\cite{McMahan16:FedLearning,McMahan:2017fedAvg}
and has recently attracted the attention of the theoretical community in the seek of the best convergence rates~\cite{Stich2018:LocalSGD,Wang2018:cooperativeSGD,yu2019parallel,basu2019qsparse,%
patel2019communication,StichK19delays,%
Li2019:decentralized,Khaled2020:local}.
Our work extends this chain and improves previous best results for convex settings and recovers the results of~\citet{Li2019:decentralized} in the non-convex case as we highlight in Table~\ref{tab:localsgd}. We  point out that all these rates are still dominated by large-batch SGD and do not match the lower bounds established in~\cite{woodworth2018graph} for the iid.\ case $\bar \zeta^2=0$. See also recent parallel work in~\cite{Woodworth20local}. Whilst these previous analysis were often specifically tailored and only applicable to the mixing structure in local SGD, our analysis is much more general and tighter at the same time.

In their recently updated parallel version, \citet{Kgoogle:cofefe} improve upon these rates by removing $\sigma$ from the second term. 
However, they do analyze a different version of local SGD (with different stepsizes for inner and outer loops) than we consider here. This change does not fit in our framework and it is not clear if similar trick is possible the decentralized setting.

\subsection{Comparison to Recent Frameworks}
We mentioned major differences to other frameworks in Section~\ref{sec:comparisontoframework} above already.
Our results for the non-convex case recover the best results from~\cite{Wang2018:cooperativeSGD} for the iid.\ case\footnote{
These results can be recovered by optimizing the stepsize in~\citep[Theorem 1]{Wang2018:cooperativeSGD} directly, instead of resorting to the worse rate stated in \citep[Corollary 1]{Wang2018:cooperativeSGD}.}
($\hat \zeta^2=0$) and the non-iid.\ case from~\cite{Li2019:decentralized} for their specific settings.
We point out that our results also cover the convex setting and deterministic setting. \looseness=-1

\subsection{Best Rates for Decentralized SGD}
We improve best known rates of Decentralized SGD \cite{Olshevsky2019:dpsgd_rate, Koloskova:2019choco} for strongly convex objectives and recover the best rates in the non-convex case \cite{Lian2017:decentralizedSGD}.

\section{Experiments}
\label{sec:experiments}
Complementing prior work that established the effectiveness of decentralized training methods~\cite{Lian2017:decentralizedSGD,Assran:2018sdggradpush} we here focus on verifying whether the numerical performance of decentralized stochastic optimization algorithms coincides with the rates predicted by theory, focusing on the strongly convex case for now.

We consider a distributed least squares objective with $f_i(\xx) := \frac{1}{2} \norm{\mA_i \xx - \bb_i }_2^2$, for fixed Hessian $\mA_i^2 = \frac{i^2}{n} \cdot \mI_d$ and sample each $\bb_i \sim \cN(0, \nicefrac{\bar\zeta^2}{ i^2} \mI_d)$ for a parameter $\bar\zeta^2$, which controls the similarity of the functions (and coincides with the parameter in Assumption~\ref{a:opt}). We control the stochastic noise $\bar\sigma^2$ by adding Gaussian noise to every stochastic gradient. We depict the effect of these parameters in Figure~\ref{fig:stepsize}.

\paragraph{Setup.} 
We consider three common network topologies, \emph{ring}, 2-$d$ \emph{torus} and \emph{fully-connected} graph and use the Metropolis-Hasting mixing matrix $W$, i.e. $w_{ij} = w_{ji} = \frac{1}{deg(i) + 1} = \frac{1}{deg(j) + 1}$ for $\{i, j\} \in E$. For all algorithms we tune the stepsize to reach a desired target accuracy $\epsilon$ with the fewest number of iterations. %

\paragraph{Discussion of Results.}
In Figure~\ref{fig:exp} we depict the results. We observe that in the high noise regime (bottom row) the graph topology and the functions similarity $\bar\zeta^2$ do not impact the number of iterations needed to reach the target accuracy (the $\frac{\bar \sigma^2}{T}$ term is dominating in this regime. We also see linear rates when $\bar \sigma^2=\bar \zeta^2 =0$ as predicted. When increasing $\bar \zeta^2$ (in the case of $\bar\sigma^2 = 0$)
we see that on the ring and torus topology the linear rate changes to a sublinear rate: even thought the curves look like straight lines, they stop converging when reaching the target accuracy (the stepsize must be further decreased to achieve higher accuracy). By comparing two top right plots, we see that for fixed topology the number of iterations increases approximately by a factor of $\sqrt{10}$ when increasing $\bar\zeta^2$ by a factor of 10, as one would expect from the term $\frac{\bar \zeta^2}{p^2T^2}$ in the convergence rate (see also Figure~\ref{fig:T2exp} in the appendix). The difference in number of iterations on the torus vs.\ ring scales approximately linear in the ratio of their mixing parameters $p$, (that is, $\Theta(n)$ as mentioned in Section~\ref{sec:p}).%

\section{Extensions}
We presented a unifying framework for the analysis of decentralized SGD methods and provide the best known convergence guarantees. Our results show that when the noise is high, decentralized SGD methods can achieve linear speedup in the number of workers $n$ and the convergence rate does only weakly depend on the graph topology, the number of local steps or the data heterogeneity. This shows that such methods are perfectly suited to solve stochastic optimization problems in a decentralized way. However, our results also reveal that when the noise is small (for e.g.\ when using large mini-batches), the effect of those parameters become more pronounced and especially function diversity can hamper the convergence of decentralized SGD methods.\looseness=-1

Our framework can be further extended by considering gradient compression techniques~\cite{Koloskova:2019choco} or 
overlapping communication steps~\cite{Assran:2018sdggradpush,Wang2020:overlap} to additionally speedup the
distributed training.

\section*{Acknowledgements}
We would like to thank Edouard Oyallon for 
indicating an inaccuracy in the proof %
in the first version of this manuscript. 
We acknowledge funding from SNSF grant 200021\_175796, as well as a Google Focused Research Award. Nicolas Loizou  acknowledges support by the IVADO Postdoctoral Funding Program.

\newpage%

\bibliographystyle{icml2020_mod}

\bibliography{papers_decentralized}

\begin{thebibliography}{90}
\providecommand{\natexlab}[1]{#1}
\providecommand{\url}[1]{\texttt{#1}}
\expandafter\ifx\csname urlstyle\endcsname\relax
  \providecommand{\doi}[1]{doi: #1}\else
  \providecommand{\doi}{doi: \begingroup \urlstyle{rm}\Url}\fi

\bibitem[Alghunaim \& Sayed(2019)Alghunaim and Sayed]{Alghunaim2019:pd}
Alghunaim, S.~A. and Sayed, A.~H.
\newblock \href{https://arxiv.org/abs/1904.01196}{Linear convergence of
  primal-dual gradient methods and their performance in distributed
  optimization}.
\newblock \emph{arXiv preprint arXiv:1904.01196}, 2019.

\bibitem[Alistarh et~al.(2017)Alistarh, Grubic, Li, Tomioka, and
  Vojnovic]{Alistarh2017:qsgd}
Alistarh, D., Grubic, D., Li, J., Tomioka, R., and Vojnovic, M.
\newblock
  \href{http://papers.nips.cc/paper/6768-qsgd-communication-efficient-sgd-via-gradient-quantization-and-encoding.pdf}{{QSGD}:
  Communication-efficient {SGD} via gradient quantization and encoding}.
\newblock In \emph{NIPS - Advances in Neural Information Processing Systems
  30}, pp.\  1709--1720. Curran Associates, Inc., 2017.

\bibitem[Arjevani \& Shamir(2015)Arjevani and Shamir]{Arjevani2015:lowerbound}
Arjevani, Y. and Shamir, O.
\newblock
  \href{http://papers.nips.cc/paper/5731-communication-complexity-of-distributed-convex-learning-and-optimization.pdf}{Communication
  complexity of distributed convex learning and optimization}.
\newblock In \emph{NIPS - Advances in Neural Information Processing Systems
  28}, pp.\  1756--1764. Curran Associates, Inc., 2015.

\bibitem[Assran \& Rabbat(2018)Assran and Rabbat]{Assran2018:async}
Assran, M. and Rabbat, M.
\newblock Asynchronous subgradient-push.
\newblock \emph{arXiv preprint arXiv:1803.08950}, 2018.

\bibitem[Assran et~al.(2019)Assran, Loizou, Ballas, and
  Rabbat]{Assran:2018sdggradpush}
Assran, M., Loizou, N., Ballas, N., and Rabbat, M.
\newblock Stochastic gradient push for distributed deep learning.
\newblock 2019.

\bibitem[Bach \& Moulines(2011)Bach and Moulines]{Moulines2011:nonasymptotic}
Bach, F.~R. and Moulines, E.
\newblock
  \href{http://papers.nips.cc/paper/4316-non-asymptotic-analysis-of-stochastic-approximation-algorithms-for-machine-learning.pdf}{Non-asymptotic
  analysis of stochastic approximation algorithms for machine learning}.
\newblock In \emph{NIPS - Advances in Neural Information Processing Systems
  24}, pp.\  451--459. Curran Associates, Inc., 2011.

\bibitem[Basu et~al.(2019)Basu, Data, Karakus, and Diggavi]{basu2019qsparse}
Basu, D., Data, D., Karakus, C., and Diggavi, S.
\newblock Qsparse-local-{SGD}: Distributed {SGD} with quantization,
  sparsification, and local computations.
\newblock \emph{arXiv preprint arXiv:1906.02367}, 2019.

\bibitem[Bottou et~al.(2018)Bottou, Curtis, and Nocedal]{Bottou2018:book}
Bottou, L., Curtis, F., and Nocedal, J.
\newblock \href{https://doi.org/10.1137/16M1080173}{Optimization methods for
  large-scale machine learning}.
\newblock \emph{SIAM Review}, 60\penalty0 (2):\penalty0 223--311, 2018.

\bibitem[Boyd et~al.(2006)Boyd, Ghosh, Prabhakar, and
  Shah]{Boyd2006:randgossip}
Boyd, S., Ghosh, A., Prabhakar, B., and Shah, D.
\newblock \href{https://doi.org/10.1109/TIT.2006.874516}{Randomized gossip
  algorithms}.
\newblock \emph{IEEE/ACM Trans. Netw.}, 14\penalty0 (SI):\penalty0 2508--2530,
  2006.

\bibitem[Chaturapruek et~al.(2015)Chaturapruek, Duchi, and
  R\'{e}]{Chaturapruek2015:noise}
Chaturapruek, S., Duchi, J.~C., and R\'{e}, C.
\newblock
  \href{http://papers.nips.cc/paper/6031-asynchronous-stochastic-convex-optimization-the-noise-is-in-the-noise-and-sgd-dont-care.pdf}{Asynchronous
  stochastic convex optimization: the noise is in the noise and {SGD} don't
  care}.
\newblock In \emph{NIPS - Advances in Neural Information Processing Systems
  28}, pp.\  1531--1539. Curran Associates, Inc., 2015.

\bibitem[Coppola(2015)]{Coppola2015:IPM}
Coppola, G.
\newblock \emph{Iterative parameter mixing for distributed large-margin
  training of structured predictors for natural language processing}.
\newblock PhD thesis, The University of Edinburgh, 2015.

\bibitem[Dean et~al.(2012)Dean, Corrado, Monga, Chen, Devin, Mao, Ranzato,
  Senior, Tucker, Yang, Le, and Ng]{dean2012large}
Dean, J., Corrado, G., Monga, R., Chen, K., Devin, M., Mao, M., Ranzato, M.,
  Senior, A., Tucker, P., Yang, K., Le, Q.~V., and Ng, A.~Y.
\newblock Large scale distributed deep networks.
\newblock In \emph{NIPS - Advances in Neural Information Processing Systems},
  pp.\  1223--1231, 2012.

\bibitem[Dekel et~al.(2012)Dekel, Gilad-Bachrach, Shamir, and
  Xiao]{Dekel2012:minibatch}
Dekel, O., Gilad-Bachrach, R., Shamir, O., and Xiao, L.
\newblock \href{http://dl.acm.org/citation.cfm?id=2503308.2188391}{Optimal
  distributed online prediction using mini-batches}.
\newblock \emph{J. Mach. Learn. Res.}, 13\penalty0 (1):\penalty0 165--202,
  January 2012.

\bibitem[Duchi et~al.(2012)Duchi, Agarwal, and
  Wainwright]{Duchi2012:distributeddualaveragig}
Duchi, J.~C., Agarwal, A., and Wainwright, M.~J.
\newblock Dual averaging for distributed optimization: Convergence analysis and
  network scaling.
\newblock \emph{IEEE Transactions on Automatic Control}, 57\penalty0
  (3):\penalty0 592--606, 2012.
\newblock \doi{10.1109/TAC.2011.2161027}.

\bibitem[Fallah et~al.(2019)Fallah, G\"{u}rb\"{u}zbalaban, Ozdaglar,
  \c{S}im\c{s}ekli, and Zhu]{Fallah2019:distributed}
Fallah, A., G\"{u}rb\"{u}zbalaban, M., Ozdaglar, A., \c{S}im\c{s}ekli, U., and
  Zhu, L.
\newblock \href{https://arxiv.org/pdf/1910.08701.pdf}{Robust distributed
  accelerated stochastic gradient methods for multi-agent networks}.
\newblock \emph{arXiv preprint arXiv:1910.08701}, 2019.

\bibitem[Fiedler(1973)]{fiedler1973algebraic}
Fiedler, M.
\newblock Algebraic connectivity of graphs.
\newblock \emph{Czechoslovak mathematical journal}, 23\penalty0 (2):\penalty0
  298--305, 1973.

\bibitem[Gower et~al.(2019)Gower, Loizou, Qian, Sailanbayev, Shulgin, and
  Richt{\'a}rik]{gower2019sgd}
Gower, R.~M., Loizou, N., Qian, X., Sailanbayev, A., Shulgin, E., and
  Richt{\'a}rik, P.
\newblock {SGD}: General analysis and improved rates.
\newblock In \emph{ICML - International Conference on Machine Learning}, pp.\
  5200--5209, 2019.

\bibitem[He et~al.(2018)He, Bian, and Jaggi]{cola2018nips}
He, L., Bian, A., and Jaggi, M.
\newblock
  \href{http://papers.nips.cc/paper/7705-cola-decentralized-linear-learning.pdf}{Cola:
  Decentralized linear learning}.
\newblock In \emph{NeurIPS - Advances in Neural Information Processing Systems
  31}, pp.\  4541--4551. 2018.

\bibitem[Iutzeler et~al.(2013)Iutzeler, Bianchi, Ciblat, and
  Hachem]{Iutzeler2013:randomizedadmm}
Iutzeler, F., Bianchi, P., Ciblat, P., and Hachem, W.
\newblock \href{https://doi.org/10.1109/CDC.2013.6760448}{Asynchronous
  distributed optimization using a randomized alternating direction method of
  multipliers}.
\newblock In \emph{Proceedings of the 52nd {IEEE} Conference on Decision and
  Control, {CDC}}, pp.\  3671--3676. {IEEE}, 2013.

\bibitem[Jakoveti\'{c} et~al.(2014)Jakoveti\'{c}, Xavier, and
  Moura]{Jakovetic2014:fast}
Jakoveti\'{c}, D., Xavier, J., and Moura, J. M.~F.
\newblock Fast distributed gradient methods.
\newblock \emph{IEEE Transactions on Automatic Control}, 59\penalty0
  (5):\penalty0 1131--1146, May 2014.

\bibitem[Johansson et~al.(2010)Johansson, Rabi, and
  Johansson]{Johansson2010:distributedsubgrad}
Johansson, B., Rabi, M., and Johansson, M.
\newblock \href{https://doi.org/10.1137/08073038X}{A randomized incremental
  subgradient method for distributed optimization in networked systems}.
\newblock \emph{SIAM Journal on Optimization}, 20\penalty0 (3):\penalty0
  1157--1170, 2010.
\newblock \doi{10.1137/08073038X}.

\bibitem[Kairouz et~al.(2019)Kairouz, McMahan, Avent, Bellet, Bennis, Bhagoji,
  Bonawitz, Charles, Cormode, Cummings, D'Oliveira, Rouayheb, Evans, Gardner,
  Garrett, Gascón, Ghazi, Gibbons, Gruteser, Harchaoui, He, He, Huo,
  Hutchinson, Hsu, Jaggi, Javidi, Joshi, Khodak, Konečný, Korolova,
  Koushanfar, Koyejo, Lepoint, Liu, Mittal, Mohri, Nock, Özgür, Pagh,
  Raykova, Qi, Ramage, Raskar, Song, Song, Stich, Sun, Suresh, Tramèr,
  Vepakomma, Wang, Xiong, Xu, Yang, Yu, Yu, and Zhao]{Kairouz2019:federated}
Kairouz, P., McMahan, H.~B., Avent, B., Bellet, A., Bennis, M., Bhagoji, A.~N.,
  Bonawitz, K., Charles, Z., Cormode, G., Cummings, R., D'Oliveira, R. G.~L.,
  Rouayheb, S.~E., Evans, D., Gardner, J., Garrett, Z., Gascón, A., Ghazi, B.,
  Gibbons, P.~B., Gruteser, M., Harchaoui, Z., He, C., He, L., Huo, Z.,
  Hutchinson, B., Hsu, J., Jaggi, M., Javidi, T., Joshi, G., Khodak, M.,
  Konečný, J., Korolova, A., Koushanfar, F., Koyejo, S., Lepoint, T., Liu,
  Y., Mittal, P., Mohri, M., Nock, R., Özgür, A., Pagh, R., Raykova, M., Qi,
  H., Ramage, D., Raskar, R., Song, D., Song, W., Stich, S.~U., Sun, Z.,
  Suresh, A.~T., Tramèr, F., Vepakomma, P., Wang, J., Xiong, L., Xu, Z., Yang,
  Q., Yu, F.~X., Yu, H., and Zhao, S.
\newblock Advances and open problems in federated learning.
\newblock \emph{arXiv preprint arXiv:1912.04977}, 2019.

\bibitem[Karimireddy et~al.(2019)Karimireddy, Kale, Mohri, Reddi, Stich, and
  Suresh]{Kgoogle:cofefe}
Karimireddy, S.~P., Kale, S., Mohri, M., Reddi, S.~J., Stich, S.~U., and
  Suresh, A.~T.
\newblock \href{https://arxiv.org/abs/1910.06378}{{SCAFFOLD}: Stochastic
  controlled averaging for on-device federated learning}.
\newblock \emph{arXiv preprint arXiv:1910.06378}, 2019.

\bibitem[Kempe et~al.(2003)Kempe, Dobra, and Gehrke]{Kempe2003:gossip}
Kempe, D., Dobra, A., and Gehrke, J.
\newblock \href{http://dl.acm.org/citation.cfm?id=946243.946317}{Gossip-based
  computation of aggregate information}.
\newblock In \emph{Proceedings of the 44th Annual IEEE Symposium on Foundations
  of Computer Science}, FOCS '03. IEEE Computer Society, 2003.

\bibitem[Khaled et~al.(2020)Khaled, Mishchenko, and
  Richt\'{a}rik]{Khaled2020:local}
Khaled, A., Mishchenko, K., and Richt\'{a}rik, P.
\newblock Tighter theory for local {SGD} on identical and heterogeneous data.
\newblock \emph{arXiv preprint arXiv:1909.04746v2}, 2020.

\bibitem[Koloskova et~al.(2019)Koloskova, Stich, and
  Jaggi]{Koloskova:2019choco}
Koloskova, A., Stich, S., and Jaggi, M.
\newblock
  \href{http://proceedings.mlr.press/v97/koloskova19a.html}{Decentralized
  stochastic optimization and gossip algorithms with compressed communication}.
\newblock In \emph{ICML - Proceedings of the 36th International Conference on
  Machine Learning}, volume~97, pp.\  3478--3487. PMLR, 2019.

\bibitem[Koloskova et~al.(2020)Koloskova, Lin, Stich, and
  Jaggi]{KoloskovaLSJ19decentralized}
Koloskova, A., Lin, T., Stich, S.~U., and Jaggi, M.
\newblock \href{https://arxiv.org/abs/1907.09356}{Decentralized deep learning
  with arbitrary communication compression}.
\newblock \emph{ICLR - International Conference on Learning Representations},
  2020.

\bibitem[Kone{\v{c}}n{\`y} et~al.(2016)Kone{\v{c}}n{\`y}, McMahan, Ramage, and
  Richt{\'a}rik]{konecny2016federated2}
Kone{\v{c}}n{\`y}, J., McMahan, H.~B., Ramage, D., and Richt{\'a}rik, P.
\newblock Federated optimization: Distributed machine learning for on-device
  intelligence.
\newblock \emph{arXiv preprint arXiv:1610.02527}, 2016.

\bibitem[Lacoste-Julien et~al.(2012)Lacoste-Julien, Schmidt, and
  Bach]{Lacoste2012:simpler}
Lacoste-Julien, S., Schmidt, M.~W., and Bach, F.~R.
\newblock \href{https://arxiv.org/abs/1212.2002}{A simpler approach to
  obtaining an {$O(1/t)$} convergence rate for the projected stochastic
  subgradient method}.
\newblock \emph{arXiv preprint arXiv:1212.2002}, 2012.

\bibitem[Lee et~al.(2015)Lee, Lin, Ma, and Yang]{lee2015distributed}
Lee, J.~D., Lin, Q., Ma, T., and Yang, T.
\newblock Distributed stochastic variance reduced gradient methods and a lower
  bound for communication complexity.
\newblock \emph{arXiv preprint arXiv:1507.07595}, 2015.

\bibitem[Lee \& Nedi{\'c}(2015)Lee and Nedi{\'c}]{lee2015asynchronous}
Lee, S. and Nedi{\'c}, A.
\newblock Asynchronous gossip-based random projection algorithms over networks.
\newblock \emph{IEEE Transactions on Automatic Control}, 61\penalty0
  (4):\penalty0 953--968, 2015.

\bibitem[Li et~al.(2018)Li, Sahu, Sanjabi, Zaheer, Talwalkar, and
  Smith]{li2018fedprox}
Li, T., Sahu, A.~K., Sanjabi, M., Zaheer, M., Talwalkar, A., and Smith, V.
\newblock On the convergence of federated optimization in heterogeneous
  networks.
\newblock \emph{arXiv preprint arXiv:1812.06127}, 2018.

\bibitem[Li et~al.(2020{\natexlab{a}})Li, Sahu, Zaheer, Sanjabi, Talwalkar, and
  Smith]{li2020feddane}
Li, T., Sahu, A.~K., Zaheer, M., Sanjabi, M., Talwalkar, A., and Smith, V.
\newblock Feddane: A federated newton-type method.
\newblock \emph{arXiv preprint arXiv:2001.01920}, 2020{\natexlab{a}}.

\bibitem[Li et~al.(2019)Li, Yang, Wang, and Zhang]{Li2019:decentralized}
Li, X., Yang, W., Wang, S., and Zhang, Z.
\newblock Communication efficient decentralized training with multiple local
  updates.
\newblock \emph{arXiv preprint arXiv:1910.09126}, 2019.

\bibitem[Li et~al.(2020{\natexlab{b}})Li, Huang, Yang, Wang, and
  Zhang]{Li2020:fedavg}
Li, X., Huang, K., Yang, W., Wang, S., and Zhang, Z.
\newblock On the convergence of {FedAvg} on non-{IID} data.
\newblock \emph{ICLR - International Conference on Learning Representations},
  openreview, 2020{\natexlab{b}}.

\bibitem[Lian et~al.(2017)Lian, Zhang, Zhang, Hsieh, Zhang, and
  Liu]{Lian2017:decentralizedSGD}
Lian, X., Zhang, C., Zhang, H., Hsieh, C.-J., Zhang, W., and Liu, J.
\newblock
  \href{http://papers.nips.cc/paper/7117-can-decentralized-algorithms-outperform-centralized-algorithms-a-case-study-for-decentralized-parallel-stochastic-gradient-descent.pdf}{Can
  decentralized algorithms outperform centralized algorithms? a case study for
  decentralized parallel stochastic gradient descent}.
\newblock In \emph{NIPS - Advances in Neural Information Processing Systems
  30}, pp.\  5330--5340. Curran Associates, Inc., 2017.

\bibitem[Lin et~al.(2020)Lin, Stich, Patel, and Jaggi]{LinSPJ2018local}
Lin, T., Stich, S.~U., Patel, K.~K., and Jaggi, M.
\newblock \href{https://arxiv.org/abs/1808.07217}{Don't use large mini-batches,
  use local {SGD}}.
\newblock \emph{ICLR - International Conference on Learning Representations},
  2020.

\bibitem[Liu et~al.(2019)Liu, Li, Anderson, and Shi]{liu2019clique}
Liu, Y., Li, B., Anderson, B.~D., and Shi, G.
\newblock Clique gossiping.
\newblock \emph{IEEE/ACM Transactions on Networking}, 27\penalty0 (6):\penalty0
  2418--2431, 2019.

\bibitem[Loizou \& Richt{\'a}rik(2016)Loizou and
  Richt{\'a}rik]{LoizouRichtarik}
Loizou, N. and Richt{\'a}rik, P.
\newblock A new perspective on randomized gossip algorithms.
\newblock In \emph{2016 IEEE Global Conference on Signal and Information
  Processing (GlobalSIP)}, pp.\  440--444. IEEE, 2016.

\bibitem[Loizou \& Richt{\'a}rik(2019)Loizou and
  Richt{\'a}rik]{loizou2019revisiting}
Loizou, N. and Richt{\'a}rik, P.
\newblock Revisiting randomized gossip algorithms: General framework,
  convergence rates and novel block and accelerated protocols.
\newblock \emph{arXiv preprint arXiv:1905.08645}, 2019.

\bibitem[Loizou et~al.(2020)Loizou, Vaswani, Laradji, and
  Lacoste-Julien]{loizou2020stochastic}
Loizou, N., Vaswani, S., Laradji, I., and Lacoste-Julien, S.
\newblock Stochastic polyak step-size for {SGD}: An adaptive learning rate for
  fast convergence.
\newblock \emph{arXiv preprint arXiv:2002.10542}, 2020.

\bibitem[Ma et~al.(2018)Ma, Bassily, and Belkin]{Ma2018:interpolation}
Ma, S., Bassily, R., and Belkin, M.
\newblock \href{http://proceedings.mlr.press/v80/ma18a.html}{The power of
  interpolation: Understanding the effectiveness of {SGD} in modern
  over-parametrized learning}.
\newblock In \emph{ICML - Proceedings of the 35th International Conference on
  Machine Learning}, volume~80, pp.\  3325--3334. PMLR, 2018.

\bibitem[McMahan et~al.(2017)McMahan, Moore, Ramage, Hampson, and
  Arcas]{McMahan:2017fedAvg}
McMahan, B., Moore, E., Ramage, D., Hampson, S., and Arcas, B. A.~y.
\newblock {Communication-Efficient Learning of Deep Networks from Decentralized
  Data}.
\newblock In \emph{AISTATS - Proceedings of the 20th International Conference
  on Artificial Intelligence and Statistics}, pp.\  1273--1282, 2017.

\bibitem[McMahan et~al.(2016)McMahan, Moore, Ramage, and
  y~Arcas]{McMahan16:FedLearning}
McMahan, H.~B., Moore, E., Ramage, D., and y~Arcas, B.~A.
\newblock \href{http://arxiv.org/abs/1602.05629}{Federated learning of deep
  networks using model averaging}.
\newblock \emph{arXiv preprint arXiv:1602.05629}, 2016.

\bibitem[{Nadiradze} et~al.(2019){Nadiradze}, {Sabour}, {Sharma}, {Markov},
  {Aksenov}, and {Alistarh}]{Nadiradze2019:PopSGD}
{Nadiradze}, G., {Sabour}, A., {Sharma}, A., {Markov}, I., {Aksenov}, V., and
  {Alistarh}, D.
\newblock {PopSGD: Decentralized Stochastic Gradient Descent in the Population
  Model}.
\newblock \emph{arXiv e-prints}, art. arXiv:1910.12308, 2019.

\bibitem[Nedi{\'c} \& Olshevsky(2014)Nedi{\'c} and
  Olshevsky]{nedic2014distributed}
Nedi{\'c}, A. and Olshevsky, A.
\newblock Distributed optimization over time-varying directed graphs.
\newblock \emph{IEEE Transactions on Automatic Control}, 60\penalty0
  (3):\penalty0 601--615, 2014.

\bibitem[Nedi\'{c} \& Olshevsky(2016)Nedi\'{c} and Olshevsky]{Nedic2016:push}
Nedi\'{c}, A. and Olshevsky, A.
\newblock Stochastic gradient-push for strongly convex functions on
  time-varying directed graphs.
\newblock \emph{IEEE Transactions on Automatic Control}, 61\penalty0
  (12):\penalty0 3936--3947, 2016.

\bibitem[Nedi\'{c} \& Ozdaglar(2009)Nedi\'{c} and
  Ozdaglar]{Nedic2009:distributedsubgrad}
Nedi\'{c}, A. and Ozdaglar, A.
\newblock Distributed subgradient methods for multi-agent optimization.
\newblock \emph{IEEE Transactions on Automatic Control}, 54\penalty0
  (1):\penalty0 48--61, 2009.

\bibitem[Nedi\'{c} et~al.(2015)Nedi\'{c}, Lee, and Raginsky]{Nedic2015:dualavg}
Nedi\'{c}, A., Lee, S., and Raginsky, M.
\newblock Decentralized online optimization with global objectives and local
  communication.
\newblock In \emph{2015 American Control Conference (ACC)}, pp.\  4497--4503,
  2015.

\bibitem[Nedi\'{c} et~al.(2016)Nedi\'{c}, Olshevsky, and Shi]{Nedic2016:drift}
Nedi\'{c}, A., Olshevsky, A., and Shi, W.
\newblock A geometrically convergent method for distributed optimization over
  time-varying graphs.
\newblock In \emph{2016 IEEE 55th Conference on Decision and Control (CDC)},
  pp.\  1023--1029, 2016.

\bibitem[Nedi\'{c} et~al.(2017)Nedi\'{c}, Olshevsky, and
  Shi]{nedic2017achieving}
Nedi\'{c}, A., Olshevsky, A., and Shi, W.
\newblock Achieving geometric convergence for distributed optimization over
  time-varying graphs.
\newblock \emph{SIAM Journal on Optimization}, 27\penalty0 (4):\penalty0
  2597--2633, 2017.

\bibitem[Nedi\'{c} et~al.(2018)Nedi\'{c}, Olshevsky, and
  Rabbat]{Nedic2018:toplogy}
Nedi\'{c}, A., Olshevsky, A., and Rabbat, M.~G.
\newblock Network topology and communication-computation tradeoffs in
  decentralized optimization.
\newblock \emph{Proceedings of the IEEE}, 106\penalty0 (5):\penalty0 953--976,
  2018.

\bibitem[Needell et~al.(2016)Needell, Srebro, and Ward]{Needell2016:sgd}
Needell, D., Srebro, N., and Ward, R.
\newblock \href{https://doi.org/10.1007/s10107-015-0864-7}{Stochastic gradient
  descent, weighted sampling, and the randomized {Kaczmarz} algorithm}.
\newblock \emph{Mathematical Programming}, 155\penalty0 (1):\penalty0 549--573,
  2016.

\bibitem[Nemirovsky \& Yudin(1983)Nemirovsky and Yudin]{nemirovskyyudin1983}
Nemirovsky, A.~S. and Yudin, D.~B.
\newblock \emph{Problem complexity and method efficiency in optimization.}
\newblock Wiley, 1983.

\bibitem[Nesterov(2004)]{Nesterov2004:book}
Nesterov, Y.
\newblock \emph{Introductory Lectures on Convex Optimization}, volume~87 of
  \emph{Springer Science \& Business Media}.
\newblock Springer US, Boston, MA, 2004.

\bibitem[Nguyen et~al.(2018)Nguyen, Nguyen, Richt\'{a}rik, Scheinberg,
  Tak\'{a}\v{c}, and van Dijk]{Nguyen2018:async}
Nguyen, L.~M., Nguyen, P.~H., Richt\'{a}rik, P., Scheinberg, K., Tak\'{a}\v{c},
  M., and van Dijk, M.
\newblock \href{https://arxiv.org/abs/1811.12403}{New convergence aspects of
  stochastic gradient algorithms}.
\newblock \emph{arXiv preprint arXiv:1811.12403}, 2018.

\bibitem[{Olshevsky} et~al.(2019){Olshevsky}, {Paschalidis}, and
  {Pu}]{Olshevsky2019:dpsgd_rate}
{Olshevsky}, A., {Paschalidis}, I.~C., and {Pu}, S.
\newblock A non-asymptotic analysis of network independence for distributed
  stochastic gradient descent.
\newblock \emph{arXiv preprint arXiv:1906.02702}, art. arXiv:1906.02702, 2019.

\bibitem[Patel \& Dieuleveut(2019)Patel and Dieuleveut]{patel2019communication}
Patel, K.~K. and Dieuleveut, A.
\newblock Communication trade-offs for synchronized distributed {SGD} with
  large step size.
\newblock \emph{arXiv preprint arXiv:1904.11325}, 2019.

\bibitem[Pu et~al.(2019)Pu, Olshevsky, and Paschalidis]{Pu2019:independent}
Pu, S., Olshevsky, A., and Paschalidis, I.~C.
\newblock A sharp estimate on the transient time of distributed stochastic
  gradient descent.
\newblock \emph{arXiv preprint arXiv:1906.02702}, 2019.

\bibitem[Rabbat(2015)]{Rabbat2015:mirrordescent}
Rabbat, M.
\newblock Multi-agent mirror descent for decentralized stochastic optimization.
\newblock In \emph{2015 IEEE 6th International Workshop on Computational
  Advances in Multi-Sensor Adaptive Processing (CAMSAP)}, pp.\  517--520, 2015.

\bibitem[Rakhlin et~al.(2012)Rakhlin, Shamir, and
  Sridharan]{Rakhlin2012:bound_for_a_0}
Rakhlin, A., Shamir, O., and Sridharan, K.
\newblock \href{http://dl.acm.org/citation.cfm?id=3042573.3042774}{Making
  gradient descent optimal for strongly convex stochastic optimization}.
\newblock In \emph{ICML - Proceedings of the 29th International Conference on
  Machine Learning}, pp.\  1571--1578. Omnipress, 2012.

\bibitem[Ram et~al.(2010)Ram, Nedi{\'c}, and Veeravalli]{ram2010asynchronous}
Ram, S.~S., Nedi{\'c}, A., and Veeravalli, V.~V.
\newblock Asynchronous gossip algorithm for stochastic optimization: Constant
  stepsize analysis.
\newblock In \emph{Recent Advances in Optimization and its Applications in
  Engineering}, pp.\  51--60. Springer, 2010.

\bibitem[Saadatniaki et~al.(2018)Saadatniaki, Xin, and
  Khan]{saadatniaki2018optimization}
Saadatniaki, F., Xin, R., and Khan, U.~A.
\newblock Optimization over time-varying directed graphs with row and
  column-stochastic matrices.
\newblock \emph{arXiv preprint arXiv:1810.07393}, 2018.

\bibitem[Scaman et~al.(2017)Scaman, Bach, Bubeck, Lee, and
  Massouli{\'e}]{Scaman2017:optimal}
Scaman, K., Bach, F., Bubeck, S., Lee, Y.~T., and Massouli{\'e}, L.
\newblock \href{http://proceedings.mlr.press/v70/scaman17a.html}{Optimal
  algorithms for smooth and strongly convex distributed optimization in
  networks}.
\newblock In \emph{ICML - Proceedings of the 34th International Conference on
  Machine Learning}, volume~70, pp.\  3027--3036. PMLR, 2017.

\bibitem[Scaman et~al.(2018)Scaman, Bach, Bubeck, Massouli\'{e}, and
  Lee]{Scaman2018:non-smooth}
Scaman, K., Bach, F., Bubeck, S., Massouli\'{e}, L., and Lee, Y.~T.
\newblock
  \href{http://papers.nips.cc/paper/7539-optimal-algorithms-for-non-smooth-distributed-optimization-in-networks.pdf}{Optimal
  algorithms for non-smooth distributed optimization in networks}.
\newblock In \emph{NeurIPS - Advances in Neural Information Processing Systems
  31}, pp.\  2745--2754. Curran Associates, Inc., 2018.

\bibitem[Schmidt \& Roux(2013)Schmidt and Roux]{Schmidt2013:fastconvergence}
Schmidt, M. and Roux, N.~L.
\newblock \href{https://arxiv.org/abs/1308.6370}{Fast convergence of stochastic
  gradient descent under a strong growth condition}.
\newblock \emph{arXiv preprint arXiv:1308.6370}, 2013.

\bibitem[Scutari \& Sun(2019)Scutari and Sun]{scutari2019distributed}
Scutari, G. and Sun, Y.
\newblock Distributed nonconvex constrained optimization over time-varying
  digraphs.
\newblock \emph{Mathematical Programming}, 176\penalty0 (1-2):\penalty0
  497--544, 2019.

\bibitem[Shamir \& Srebro(2014)Shamir and Srebro]{Shamir2014:distributedSO}
Shamir, O. and Srebro, N.
\newblock Distributed stochastic optimization and learning.
\newblock \emph{2014 52nd Annual Allerton Conference on Communication, Control,
  and Computing (Allerton)}, pp.\  850--857, 2014.

\bibitem[Shi et~al.(2015)Shi, Ling, Wu, and Yin]{shi2015extra}
Shi, W., Ling, Q., Wu, G., and Yin, W.
\newblock {EXTRA}: An exact first-order algorithm for decentralized consensus
  optimization.
\newblock \emph{SIAM Journal on Optimization}, 25\penalty0 (2):\penalty0
  944--966, 2015.

\bibitem[Stich(2019{\natexlab{a}})]{Stich19sgd}
Stich, S.~U.
\newblock \href{https://arxiv.org/abs/1907.04232}{Unified optimal analysis of
  the (stochastic) gradient method}.
\newblock \emph{arXiv preprint arXiv:1907.04232}, 2019{\natexlab{a}}.

\bibitem[Stich(2019{\natexlab{b}})]{Stich2018:LocalSGD}
Stich, S.~U.
\newblock \href{https://arxiv.org/abs/1805.09767}{Local {SGD} converges fast
  and communicates little}.
\newblock \emph{ICLR - International Conference on Learning Representations},
  art. arXiv:1805.09767, 2019{\natexlab{b}}.

\bibitem[Stich \& Karimireddy(2019)Stich and Karimireddy]{StichK19delays}
Stich, S.~U. and Karimireddy, S.~P.
\newblock \href{https://arxiv.org/abs/1909.05350}{The error-feedback framework:
  Better rates for {SGD} with delayed gradients and compressed communication}.
\newblock \emph{arXiv preprint arXiv:1909.05350}, 2019.

\bibitem[Stich et~al.(2018)Stich, Cordonnier, and
  Jaggi]{Stich2018:sparsifiedSGD}
Stich, S.~U., Cordonnier, J.-B., and Jaggi, M.
\newblock
  \href{http://papers.nips.cc/paper/7697-sparsified-sgd-with-memory.pdf}{Sparsified
  {SGD} with memory}.
\newblock In \emph{NeurIPS - Advances in Neural Information Processing Systems
  31}, pp.\  4452--4463. 2018.

\bibitem[Tang et~al.(2018{\natexlab{a}})Tang, Gan, Zhang, Zhang, and
  Liu]{Tang2018:decentralized}
Tang, H., Gan, S., Zhang, C., Zhang, T., and Liu, J.
\newblock
  \href{http://papers.nips.cc/paper/7992-communication-compression-for-decentralized-training.pdf}{Communication
  compression for decentralized training}.
\newblock In \emph{NeurIPS - Advances in Neural Information Processing Systems
  31}, pp.\  7663--7673. Curran Associates, Inc., 2018{\natexlab{a}}.

\bibitem[Tang et~al.(2018{\natexlab{b}})Tang, Lian, Yan, Zhang, and
  Liu]{Tang2018:d2}
Tang, H., Lian, X., Yan, M., Zhang, C., and Liu, J.
\newblock \href{http://proceedings.mlr.press/v80/tang18a.html}{{D}$^2$:
  Decentralized training over decentralized data}.
\newblock In \emph{ICML - Proceedings of the 35th International Conference on
  Machine Learning}, volume~80, pp.\  4848--4856. PMLR, 2018{\natexlab{b}}.

\bibitem[Tang et~al.(2019)Tang, Lian, Qiu, Yuan, Zhang, Zhang, and
  Liu]{Tang2019:squeeze}
Tang, H., Lian, X., Qiu, S., Yuan, L., Zhang, C., Zhang, T., and Liu, J.
\newblock \href{https://arxiv.org/abs/1907.07346}{Deepsqueeze: Decentralization
  meets error-compensated compression}.
\newblock \emph{arXiv preprint arXiv:1907.07346}, 2019.

\bibitem[Tsitsiklis(1984)]{Tsitsiklis1985:gossip}
Tsitsiklis, J.~N.
\newblock \emph{Problems in decentralized decision making and computation}.
\newblock PhD thesis, Massachusetts Institute of Technology, 1984.

\bibitem[Uribe et~al.(2018)Uribe, Lee, and Gasnikov]{Uribe:2018uk}
Uribe, C.~A., Lee, S., and Gasnikov, A.
\newblock A dual approach for optimal algorithms in distributed optimization
  over networks.
\newblock \emph{arXiv preprint arXiv:1809.00710}, 2018.

\bibitem[Wang \& Joshi(2018)Wang and Joshi]{Wang2018:cooperativeSGD}
Wang, J. and Joshi, G.
\newblock \href{http://arxiv.org/abs/1808.07576}{Cooperative {SGD:} {A} unified
  framework for the design and analysis of communication-efficient {SGD}
  algorithms}.
\newblock \emph{arXiv preprint arXiv:1808.07576}, 2018.

\bibitem[Wang et~al.(2019)Wang, Sahu, Yang, Joshi, and Kar]{Wang2019:matcha}
Wang, J., Sahu, A.~K., Yang, Z., Joshi, G., and Kar, S.
\newblock \href{http://arxiv.org/abs/1905.09435}{{MATCHA:} speeding up
  decentralized {SGD} via matching decomposition sampling}.
\newblock \emph{arXiv preprint arXiv:1905.09435}, 2019.

\bibitem[Wang et~al.(2020)Wang, Liang, and Joshi]{Wang2020:overlap}
Wang, J., Liang, H., and Joshi, G.
\newblock Overlap local-{SGD}: An algorithmic approach to hide communication
  delays in distributed {SGD}.
\newblock \emph{manuscript}, 2020.

\bibitem[Wei \& Ozdaglar(2012)Wei and Ozdaglar]{Wei2012:distributedadmm}
Wei, E. and Ozdaglar, A.
\newblock Distributed alternating direction method of multipliers.
\newblock In \emph{2012 IEEE 51st IEEE Conference on Decision and Control
  (CDC)}, pp.\  5445--5450, 2012.

\bibitem[Woodworth et~al.(2020)Woodworth, Patel, Stich, Dai, Bullins, McMahan,
  Shamir, and Srebro]{Woodworth20local}
Woodworth, B., Patel, K.~K., Stich, S.~U., Dai, Z., Bullins, B., McMahan,
  H.~B., Shamir, O., and Srebro, N.
\newblock \href{https://arxiv.org/abs/2002.07839}{Is local {SGD} better than
  minibatch {SGD}?}
\newblock \emph{arXiv preprint arXiv:2002.07839}, 2020.

\bibitem[Woodworth et~al.(2018)Woodworth, Wang, Smith, McMahan, and
  Srebro]{woodworth2018graph}
Woodworth, B.~E., Wang, J., Smith, A., McMahan, H.~B., and Srebro, N.
\newblock Graph oracle models, lower bounds, and gaps for parallel stochastic
  optimization.
\newblock In \emph{NeurIPS - Advances in Neural Information Processing
  Systems}, pp.\  8496--8506, 2018.

\bibitem[Xi \& Khan(2017)Xi and Khan]{xi2017dextra}
Xi, C. and Khan, U.~A.
\newblock Dextra: A fast algorithm for optimization over directed graphs.
\newblock \emph{IEEE Transactions on Automatic Control}, 62\penalty0
  (10):\penalty0 4980--4993, 2017.

\bibitem[Xiao \& Boyd(2004)Xiao and Boyd]{Xiao2014:averaging}
Xiao, L. and Boyd, S.
\newblock
  \href{http://www.sciencedirect.com/science/article/pii/S0167691104000398}{Fast
  linear iterations for distributed averaging}.
\newblock \emph{Systems \& Control Letters}, 53\penalty0 (1):\penalty0 65--78,
  2004.

\bibitem[Yu et~al.(2019)Yu, Yang, and Zhu]{yu2019parallel}
Yu, H., Yang, S., and Zhu, S.
\newblock Parallel restarted {SGD} with faster convergence and less
  communication: Demystifying why model averaging works for deep learning.
\newblock In \emph{Proceedings of the AAAI Conference on Artificial
  Intelligence}, volume~33, pp.\  5693--5700, 2019.

\bibitem[Zhao et~al.(2018)Zhao, Li, Lai, Suda, Civin, and
  Chandra]{zhao2018federated}
Zhao, Y., Li, M., Lai, L., Suda, N., Civin, D., and Chandra, V.
\newblock Federated learning with non-iid data.
\newblock \emph{arXiv preprint arXiv:1806.00582}, 2018.

\bibitem[Zhou \& Cong(2018)Zhou and Cong]{Zhou2018:local}
Zhou, F. and Cong, G.
\newblock \href{https://doi.org/10.24963/ijcai.2018/447}{On the convergence
  properties of a k-step averaging stochastic gradient descent algorithm for
  nonconvex optimization}.
\newblock In \emph{Proceedings of the Twenty-Seventh International Joint
  Conference on Artificial Intelligence, {IJCAI-18}}, pp.\  3219--3227.
  International Joint Conferences on Artificial Intelligence Organization, 7
  2018.

\bibitem[Zinkevich et~al.(2010)Zinkevich, Weimer, Li, and
  Smola]{zinkevich2010parallelized}
Zinkevich, M., Weimer, M., Li, L., and Smola, A.~J.
\newblock Parallelized stochastic gradient descent.
\newblock In \emph{NIPS - Advances in Neural Information Processing Systems},
  pp.\  2595--2603, 2010.

\end{thebibliography}
\normalsize

\appendix
\onecolumn
\icmltitle{APPENDIX \\ A Unified Theory of Decentralized SGD \\with Changing Topology and Local Updates}

The appendix is organized as follows: 
In Section~\ref{sec:splitting}, we rewrite Algorithm~\ref{alg:random_W_decentr_sgd_matrix} equivalently in matrix notation as Algorithm~\ref{alg:decentr_sgd_matrix} and give a sketch of the proof using this new notation.
In Section~\ref{sec:appendixtechnical} we state a few auxiliary technical lemmas, before giving  all details for the proof of the theorem in Sections~\ref{sec:descentconsensus} and \ref{sec:mainrecursionlemmas}. %
We conclude the appendix in Section~\ref{sec:additionalexp} by presenting additional numerical results that confirm the tightness of our theoretical analysis in the strongly convex case.

\section{Proof of Theorem~\ref{thm:summary}}
\label{sec:splitting}
\subsection{Decentralized SGD in Matrix Notation}
We can rewrite Algorithm \ref{alg:random_W_decentr_sgd_matrix} using the following matrix notation, extending the definition used in the main text:
\begin{align}
\begin{split}
X^{(t)} &:= \left[ \xx_1^{(t)},\dots, \xx_n^{(t)}\right] \in \R^{d\times n}, \\ \bar{X}^{(t)} &:= \left[ \bar{\xx}^{(t)},\dots, \bar{\xx}^{(t)}\right]  \in \R^{d\times n}, \\  \partial F(X^{(t)}, \xi^{(t)}) &:= \left[\nabla F_1(\xx_{1}^{(t)}, \xi_1^{(t)}), \dots,  \nabla F_n(\xx_{n}^{(t)}, \xi_n^{(t)})\right]  \in \R^{d\times n}. \label{eq:matrix_notation}
\end{split}
\end{align}
\begin{algorithm}[H]
	\caption{\textsc{Decentralized SGD (matrix notation)}}\label{alg:decentr_sgd_matrix}
	\begin{algorithmic}[1]
		\INPUT{: $X^{(0)}$, stepsizes $\{\eta_t\}_{t=0}^{T-1}$, number of iterations $T$, mixing matrix distributions $\cW^{(t)}$ for $t \in [0, T ]$}
		\FOR{$t$\textbf{ in} $0\dots T$}
		\STATE Sample $W^{(t)} \sim \cW^{(t)}$ 
		\STATE $X^{(t + \frac{1}{2})} = X^{(t)} - \eta_t\partial F_i(X^{(t)}, \xi_i^{(t)})$ \hfill $\triangleright$ stochastic gradient updates
		\STATE $X^{(t + 1)} = X^{(t + \frac{1}{2})} W^{(t)} $ \hfill  $\triangleright$ gossip averaging
		\ENDFOR
	\end{algorithmic}
\end{algorithm}

\subsection{Proof Sketch---Combining Consensus Progress (Gossip) and Optimization Progress (SGD)}
In this section we sketch of the proof for Theorem~\ref{thm:summary}. 
As a first step in the proof, we will derive an upper bound on the expected progress, measured as distance to the optimum, $r_t =\E{\norm{\bar{\xx}^{(t)} - \xx^\star}}^2$ for the convex cases, and function suboptimality $r_t = \E f(\bar \xx^{(t)}) - f^\star$ in the non-convex case. 
These bounds will have the following form:
\begin{align}\label{eq:rec1}
r_{t + 1} \leq (1 - a \eta_t) r_t - b \eta_t e_t + c \eta_t^2 +  \eta_t B \Xi_t \,,
\end{align}
with $\Xi_t = \frac{1}{n}\E_t{\sum_{i=1}^n  \norm{\xx_i^{(t)}-\bar \xx^{(t)}}^2}$ and 
\begin{itemize}
	\item for both convex cases $r_t =\E{\norm{\bar{\xx}^{(t)} - \xx^\star}}^2$, $e_t = f(\bar\xx^{(t)}) - f(\xx^\star)$, $a = \frac{\mu}{2}$, $b  =1$, $c = \frac{\bar\sigma^2}{n}$,  $B = 3L$ (Lemma~\ref{lem:stochastic_avg});
	\item for the non-convex case $r_t = \E f(\bar \xx^{(t)}) - f^\star$, $ e_t = \norm{ \nabla f(\bar \xx^{(t)}) }_2^2$, $a = 0$, $b  = \frac{1}{4}$, $c = \frac{L \hat{\sigma}^2}{n}$,  $B = L^2$ (Lemma~\ref{lem:decrease_nc}).
\end{itemize}
We will then bound the consensus distance $\Xi_{t}$ as detailed in Section \ref{sec:descentconsensus}; Lemmas~\ref{lem:consensus} and~\ref{lem:consensus_nc} by a recursion of the form 
\begin{align}
\Xi_{t} & \leq \left( 1 - \frac{p}{2}\right) \Xi_{m\tau} + \frac{p}{64 \tau}  \sum_{j = m\tau}^{t - 1} \Xi_j + D \sum_{j = m\tau}^{t - 1} \eta_j^2 e_j + A \sum_{j = m\tau}^{t - 1} \eta_j^2, \label{eq:consensus_recursion}
\end{align}
where $m = \lfloor {t / \tau} \rfloor  - 1$; for convex cases $A = 8 \bar\sigma^2  + \frac{18\tau}{p} \bar\zeta^2$ , $D = 72 L \frac{\tau}{p}$ (Lemma~\ref{lem:consensus}) and for non-convex case $A = 2 \hat \sigma^2  + 2\left(\frac{6 \tau}{p} + M \right)\hat \zeta^2$, $D = 2 P \left( \frac{6 \tau}{p} + M\right)$ (Lemma~\ref{lem:consensus_nc}). 

Note that \eqref{eq:consensus_recursion} holds only for $t \geq (m + 1)\tau$. To be able to simplify \eqref{eq:consensus_recursion} we additionally consider $m \tau \leq  t < (m + 1)\tau$ and prove (Lemmas~\ref{lem:consensus2}, \ref{lem:consensus_nc2}) that with the same parameters as above, it holds
\begin{align}
\Xi_{t} & \leq \left( 1 + \frac{p}{2}\right) \Xi_{m\tau} + \frac{p}{64 \tau}  \sum_{j = m\tau}^{t - 1} \Xi_j + D \sum_{j = m\tau}^{t - 1} \eta_j^2 e_j + A \sum_{j = m\tau}^{t - 1} \eta_j^2, \label{eq:consensus_recursion2}
\end{align}

Next, we simplify this recursive equation \eqref{eq:consensus_recursion} using Lemma~\ref{lem:solve_consensus_recursion} and some positive weights $\{w_t\}_{t \geq 0}$ (see Lemma~\ref{lem:solve_consensus_recursion} for the definition of the weights $w_t$) to 
\begin{align}\label{eq:rec2}
B \cdot \sum_{t = 0}^T w_t \Xi_{t} \ \leq\  \frac{b}{2} \cdot \sum_{t = 0}^T w_t e_t + 64 A B \frac{\tau}{p}\cdot \sum_{t = 0}^T w_t \eta_t^2  \,,
\end{align}
where again $\Xi_t = \frac{1}{n}\E_t{\sum_{i=1}^n  \norm{\xx_i^{(t)}-\bar \xx^{(t)}}^2}$. %

Then we combine \eqref{eq:rec1} and \eqref{eq:rec2}. Firstly rearranging \eqref{eq:rec1}, multiplying by $w_t$ and dividing by $\eta_t$, we get 
\begin{align*}
b w_t e_t  \leq \frac{(1 - a \eta_t)}{\eta_t} w_t r_t - \frac{w_t }{\eta_t} r_{t + 1} + c w_t \eta_t + B w_t \Xi_t \,,
\end{align*}
Now summing up and dividing by $W_T = \sum_{t = 0}^T w_t$,
\begin{align*}
\frac{1}{W_T}\sum_{t = 0}^T b w_t e_t  & \leq \frac{1}{W_T} \sum_{t=0}^T \left(\frac{(1 - a \eta_t)w_t}{\eta_t}  r_t - \frac{w_t }{\eta_t} r_{t + 1} \right) + \frac{c}{W_T} \sum_{t = 0}^Tw_t \eta_t + \frac{1}{W_T} B \sum_{t = 0}^T w_t \Xi_t \\
& \stackrel{\eqref{eq:rec2}}{\leq}  \frac{1}{W_T} \sum_{t=0}^T \left(\frac{(1 - a \eta_t)w_t}{\eta_t}  r_t - \frac{w_t }{\eta_t} r_{t + 1} \right)  + \frac{c}{W_T} \sum_{t = 0}^Tw_t\eta_t + \frac{1}{2W_T}  \sum_{t = 0}^T w_t e_t + \frac{64 B A}{W_T} \frac{\tau}{p}\sum_{t = 0}^T w_t \eta_t^2 \,,
\end{align*}
Therefore,
\begin{align} \label{eq:rec3}
\frac{1}{2 W_T}\sum_{t = 0}^T bw_t e_t \leq \frac{1}{W_T} \sum_{t=0}^T \left(\frac{(1 - a \eta_t)w_t}{\eta_t}  r_t - \frac{w_t }{\eta_t} r_{t + 1} \right)  + \frac{c}{W_T} \sum_{t = 0}^Tw_t\eta_t + \frac{64 B A}{W_T} \frac{\tau}{p} \sum_{t = 0}^T w_t \eta_t^2
\end{align}

Finally, to solve this main recursion \eqref{eq:rec3} and obtain the final convergence rates of Theorem~\ref{thm:summary}, we will use the following Lemmas, which will be presented in Section \ref{sec:mainrecursionlemmas}:
\begin{itemize}
	\item Lemma~\ref{lem:rate_strongly_convex} for strongly convex case when $a > 0$.
	\item Lemmas~\ref{lem:rate_weakly_convex} and \ref{lem:tuning_stepsize} for both weakly convex and non-convex cases as their common feature is that $a = 0$.
\end{itemize}

\subsection{How the Proof of Theorem~\ref{thm:summary} Follows}
\label{sec:summary_proof_main_th}
In this section we summarize how the proof of Theorem~\ref{thm:summary} follows from the results that we establish in Sections~\ref{sec:descentconsensus} and \ref{sec:mainrecursionlemmas} below.
Note that for convex cases we require both $f_i$ and $F_i$ to be convex as in Lemma~\ref{lem:consensus}.

\begin{proof}[Proof of Theorem~\ref{thm:summary}, strongly convex case]
	The proof follows by applying the result of Lemma~\ref{lem:rate_strongly_convex} to the equation \eqref{eq:rec3} (obtained with Lemmas~\ref{lem:stochastic_avg}, \ref{lem:consensus}, \ref{lem:solve_consensus_recursion}) with $r_t =\E{\norm{\bar{\xx}^{(t)} - \xx^\star}}^2$, $e_t = f(\bar\xx^{(t)}) - f(\xx^\star)$, $a = \frac{\mu}{2}$, $b = 1$, $c = \frac{\bar\sigma^2}{n}$, $d = \frac{96 \sqrt{3}\tau  L}{p}$, $A = \bar\sigma^2  + \frac{18\tau}{p} \bar\zeta^2$, $B = 3L$, $D = 72 L \frac{\tau}{p}$.
	It is only left to show that chosen weights $w_t$ stepsizes $\eta_t$ in Lemma~\ref{lem:rate_strongly_convex} satisfy conditions of Lemmas~\ref{lem:stochastic_avg}, \ref{lem:consensus}, \ref{lem:solve_consensus_recursion}. It is shown in Proposition~\ref{prop:examples-tau-slow} that $\{\eta_t\}$ is $\frac{8\tau}{p}$-slow decreasing and $\{w_t\}$ is $\frac{16\tau}{p}$-slow increasing (condition in Lemma~\ref{lem:solve_consensus_recursion}). Moreover the stepsize $\eta_t := \eta < \frac{1}{d}$ is smaller than conditions on $\eta_t$ in Lemmas~\ref{lem:stochastic_avg}, \ref{lem:consensus}, \ref{lem:solve_consensus_recursion}. 
\end{proof}

	\begin{proof}[Proof of Theorem~\ref{thm:summary}, weakly convex case] 
	The proof follows by applying the result of Lemma~\ref{lem:rate_weakly_convex} to the equation \eqref{eq:rec3} (obtained with Lemmas~\ref{lem:stochastic_avg}, \ref{lem:consensus}, \ref{lem:solve_consensus_recursion}) with $r_t =\E{\norm{\bar{\xx}^{(t)} - \xx^\star}}^2$, $e_t = f(\bar\xx^{(t)}) - f(\xx^\star)$, $a = 0$, $b = 1$, $c = \frac{\bar\sigma^2}{n}$, $d = \frac{96 \sqrt{6}\tau  L}{p}$, $A = 8\bar\sigma^2  + \frac{18\tau}{p} \bar\zeta^2$, $B = 3L$, $D = 72 L \frac{\tau}{p}$.
	It is shown in Proposition~\ref{prop:examples-tau-slow} that $\{\eta_t\}$ and $\{w_t\}$ chosen in Lemma~\ref{lem:rate_weakly_convex} satisfy condition in Lemma~\ref{lem:solve_consensus_recursion}: $\{\eta_t\}$ is $\frac{8\tau}{p}$-slow decreasing and $\{w_t\}$ is $\frac{16\tau}{p}$-slow increasing. Moreover the stepsize $\eta_t := \eta < \frac{1}{d}$ is smaller than conditions on $\eta_t$ in Lemmas~\ref{lem:stochastic_avg}, \ref{lem:consensus}, \ref{lem:solve_consensus_recursion}. 
\end{proof}

\begin{proof}[Proof of Theorem~\ref{thm:summary}, non-convex case]
	applying the result of Lemma~\ref{lem:rate_weakly_convex} to the equation \eqref{eq:rec3} (obtained with Lemmas~\ref{lem:decrease_nc}, \ref{lem:consensus_nc}, \ref{lem:solve_consensus_recursion}) with $r_t = \E f(\bar \xx^{(t)}) - f^\star$, $ e_t = \norm{ \nabla f(\bar \xx^{(t)}) }_2^2$, $a = 0$, $b  = \frac{1}{4}$, $c = \frac{L \hat{\sigma}^2}{n}$, $d = 64 L \sqrt{2 \max\{P, 1\} \left( \frac{6 \tau}{p} + M\right)\frac{\tau}{p}}$,  $A = 2 \hat \sigma^2  + 2\left(\frac{6 \tau}{p} + M \right)\hat \zeta^2$, $B = L^2$, $D = 2 P \left( \frac{6 \tau}{p} + M\right)$. 
	Weights $w_t$ stepsizes $\eta_t$ chosen in Lemma~\ref{lem:rate_weakly_convex} satisfy conditions of Lemmas~\ref{lem:decrease_nc}, \ref{lem:consensus_nc}, \ref{lem:solve_consensus_recursion}, as shown in Proposition~\ref{prop:examples-tau-slow}.
\end{proof}

\subsection{Improved rate when $\tau = 1$ (recovering mini-batch SGD convergence results)}
	In the special case when $\tau = 1$ the proof can be simplified and the rate can be improved: there will be an additional $(1-p)$ factor appearing in the middle term, e.g in strongly convex case the improved rate reads as
	\begin{align*}
	\tilde\cO \left(\frac{\bar\sigma^2}{n \mu T} + \frac{L (\bar\zeta^2 + p \bar\sigma^2) (1 - p)}{\mu^2 p^2 T^2} + \frac{LR_0^2}{p} \exp\left[- \frac{\mu T p}{L}\right] \right)\,.
	\end{align*}
	The main difference to the general result stated in Theorem~\ref{thm:summary} (for $\tau \geq 1$) is that the second term is multiplied with $(1 - p)$, allowing to recover the rate of mini-batch SGD in the case of fully-connected graph when $p = 1$. This improvement also holds for the weakly-convex and non-convex case.
	
	In order to do so, one has to observe that the consensus distance Lemmas~\ref{lem:consensus} and \ref{lem:consensus_nc} can be improved when $\tau = 1$. In the first lines of both these proofs we multiply with $(1 - p)$ not only the first term $\norm{X^{(t)} - \bar{X}^{(t)}}_2^2$ but also the second term with the gradient as during the 1-step averaging both $\xx^{(t)}$ and $\eta_t\partial F_i(X^{(t)}, \xi_i^{(t)})$ are averaged with mixing matrix $W^{(t)}$ (line 4 of Algorithm~\ref{alg:decentr_sgd_matrix}). We omit the full derivations for this special case, as they can easily be obtained by following the current proofs.

\section{Technical Preliminaries}
\label{sec:appendixtechnical}
\subsection{Implications of the assumptions}

\begin{proposition}\label{rem:average_preserve}
	One step of gossip averaging with the mixing matrix $W$ (def.~\ref{def:valid_mixing}) preserves the average of the iterates, i.e. 
	\begin{align*}
	XW \frac{\1\1^\top}{n} = X \frac{\1\1^\top}{n} \,.
	\end{align*}
\end{proposition}

\begin{proposition}[Implications of the smoothness Assumption~\ref{a:lsmooth}]
	If for functions $F_i(\xx, \xi)$ Assumption~\ref{a:lsmooth} holds, then it also holds that 
	\begin{align}
	F_i(\xx, \xi) \leq F_i(\yy, \xi) + \lin{\nabla F_i(\yy, \xi),\xx-\yy} + \frac{L}{2}\norm{\xx-\yy}^2_2, \qquad \forall \xx, \yy \in \R^d, \xi \in \Omega_i \label{eq:F-lsmooth}
	\end{align}
	If functions $f_i(\xx) = \E_\xi F_i(\xx, \xi)$, then 
	\begin{align}
	f_i(\xx) \leq f_i(\yy) + \lin{\nabla f_i(\yy),\xx-\yy} + \frac{L}{2}\norm{\xx-\yy}^2_2, \quad \forall \xx, \yy \in \R^d \label{eq:f-lsmooth}
	\end{align}
	Moreover, if in addition $F_i$ are convex functions, then 
	\begin{align}
	\|\nabla  f_i(\xx) - \nabla f_i(\yy) \|_2 &\leq L \norm{\xx-\yy}_2, && \forall \xx, \yy \in \R^d, \label{eq:f-lsmooth-convex}\\
	\norm{ \nabla g(\xx) - \nabla g(\yy)}_2^2 &\leq 2L \left(g(\xx)- g(\yy) - \lin{\xx-\yy,\nabla g(\yy)} \right)\,, && \forall \xx, \yy \in \R^d, \label{eq:lsmooth_norm}
	\end{align}
	where $g(\xx)$ is either $F_i$ or $f_i$. 
\end{proposition}

\begin{proposition}[Implications of the smoothness Assumption~\ref{a:lsmooth_nc}]
	From Assumption~\ref{a:lsmooth_nc} it follows that 
	\begin{align}
	f_i(\xx) \leq f_i(\yy) + \lin{\nabla f_i(\yy),\xx-\yy} + \frac{L}{2}\norm{\xx-\yy}^2_2, \quad \forall \xx, \yy \in \R^d \,. %
	\end{align}
\end{proposition}

\subsection{Useful Inequalities}

\begin{lemma}\label{remark:norm_of_sum}
	For arbitrary set of $n$ vectors $\{\aa_i\}_{i = 1}^n$, $\aa_i \in \R^d$
	\begin{equation}\label{eq:norm_of_sum}
	\norm{\sum_{i = 1}^n \aa_i}^2 \leq n \sum_{i = 1}^n \norm{\aa_i}^2 \,.
	\end{equation}
\end{lemma}
\begin{lemma}\label{remark:scal_product}
	For given two vectors $\aa, \bb \in \R^d$
	\begin{align}\label{eq:scal_product}
	&2\lin{\aa, \bb} \leq \gamma \norm{\aa}^2 + \gamma^{-1}\norm{\bb}^2\,, & &\forall \gamma > 0 \,.
	\end{align}
\end{lemma}
\begin{lemma}\label{remark:norm_of_sum_of_two}
	For given two vectors $\aa, \bb \in \R^d$ %
	\begin{align}\label{eq:norm_of_sum_of_two}
	\norm{\aa + \bb}^2 \leq (1 + \alpha)\norm{\aa}^2 + (1 + \alpha^{-1})\norm{\bb}^2,\,\, & &\forall \alpha > 0\,.
	\end{align}
	This inequality also holds for the sum of two matrices $A,B \in \R^{n \times d}$ in Frobenius norm.
\end{lemma}

\begin{remark}\label{rem:frobenious_norm_of_matrix_mult}
	For $A\in \R^{d\times n}$, $B\in \R^{n\times n}$
	\begin{align}\label{eq:frob_norm_of_multiplication}
	\norm{AB}_F \leq \norm{A}_F \norm{B}_2 \,.
	\end{align}
\end{remark}

\subsection{$\tau$-slow Sequences}
\begin{definition}[$\tau$-slow sequences \cite{StichK19delays}]\label{def:tau-slow}
	The sequence $\{a_t\}_{t \geq 0}$ of positive values is \emph{$\tau$-slow decreasing} for parameter $\tau > 0$ if
	\begin{align*}
	a_{t + 1} \leq a_t,\quad  \forall t \geq 0 && \text{and}, && a_{t + 1}\left( 1 + \frac{1}{2\tau}\right) \geq a_t, \quad\forall t \geq 0 \,.
	\end{align*}
	The sequence $\{a_t\}_{t \geq 0}$ is \emph{$\tau$-slow increasing} if $\{a_t^{-1}\}_{t \geq 0}$ is $\tau$-slow decreasing.
\end{definition}
\begin{proposition}[Examples]\label{prop:examples-tau-slow}\hfill\null
	\begin{enumerate}
		\item The sequence $\{\eta_t^2\}_{t\geq 0}$ with $\eta_t = \frac{a}{b + t}$, $b \geq \frac{32}{p}$ is $\frac{4}{p}$-slow decreasing. 
		\item The sequence of constant stepsizes $\{\eta_t^2\}_{t\geq 0}$ with $\eta_t = \eta$ is $\tau$-slow decreasing for any $\tau$.
		\item The sequence $\{w_t \}_{t \geq 0}$ with $w_t = (1 - \frac{p}{64 \tau c})^{-t}$, $c \geq 1$ is $\frac{16 \tau}{p}$-slow increasing. 
		\item The sequence $\{w_t \}_{t \geq 0}$ with $w_t = (b + t)^2$, $b \geq \frac{128}{p}$ is $\frac{16}{p}$-slow increasing. 
		\item The sequence of constant weights $\{w_t\}_{t\geq 0}$ with $w_t = 1$ is $\tau$-slow increasing for any $\tau$.
	\end{enumerate}
\end{proposition}

\section{Descent Lemmas and Consensus Recursions}
\label{sec:descentconsensus}
In this section, according to our proof sketch we derive descent \eqref{eq:rec1} and consensus recursions \eqref{eq:rec2} for both convex and also non-convex cases.

\subsection{Convex Cases}
We require both $f_i$ and $F_i$ to be convex.
We do not need Assumption~\ref{a:strong} to hold for all $\xx, \yy \in \R^d$ and we could weaken it to hold only for $\xx = \xx^\star$ and for all $\yy \in \R^d$. 

\begin{proposition}[Mini-batch variance]\label{rem:variance} Let functions $F_i(\xx, \xi)$ , $i \in [n]$ be $L$-smooth (Assumption~\ref{a:lsmooth}) with bounded noise at the optimum (Assumption~\ref{a:opt}). Then for any $\xx_i \in \R^d, i \in [n]$ and $\bar \xx := \frac{1}{n}\sum_{i=1}^n \xx_i$ it holds
	\begin{align*}\EE{\xi_1, \dots, \xi_n}{\norm{\frac{1}{n}\sum_{i = 1}^{n} \left(\nabla f_i(\xx_i) - \nabla F_i (\xx_i, \xi_i)\right)}}^2\leq \frac{3L^2}{n^2} \sum_{i = 1}^{n}  \norm{ \xx_i - \bar{\xx}}^2 + \frac{6L}{n}\left(f(\bar{\xx}) - f(\xx^\star)\right) + \frac{3 \bar{\sigma}^2 }{n} \,.
	\end{align*}
\end{proposition}
\begin{proof}
	\begin{align*}
	\EE{\xi_1, \dots, \xi_n}{}&\norm{\frac{1}{n}\sum_{i = 1}^{n} \left(\nabla f_i(\xx_i) - \nabla F_i (\xx_i, \xi_i)\right)}^2 
	\leq  \frac{1}{n^2} \sum_{i = 1}^{n} \EE{\xi_i}{ \norm{ \nabla F_i(\xx_i, \xi_i) - \nabla f_i(\xx_i)}^2}\\
	& \leq \frac{3}{n^2} \sum_{i = 1}^{n} \EE{\xi_i}{\Big( \norm{ \nabla F_i(\xx_i, \xi_i) - \nabla F_i(\bar\xx, \xi_i) - \nabla f_i(\xx_i) + \nabla f_i(\bar\xx)}^2  }\\
	& \qquad\qquad + \norm{\nabla F_i(\bar\xx, \xi_i) - \nabla F_i(\xx^\star, \xi_i) - \nabla f_i(\bar\xx^{(t)}) + \nabla f_i(\xx^\star)}^2  + \norm{ \nabla F_i(\xx^\star, \xi_i) - \nabla f_i(\xx^\star)}^2 \Big)\\
	& \leq \frac{3}{n^2} \sum_{i = 1}^{n} \EE{\xi_i}{\left( \norm{ \nabla F_i(\xx_i^{(t)}, \xi_i^{(t)}) - \nabla F_i(\bar\xx, \xi_i)}^2  + \norm{\nabla F_i(\bar\xx, \xi_i) - \nabla F_i(\xx^\star, \xi_i)}^2  + \norm{ \nabla F_i(\xx^\star, \xi_i) - \nabla f_i(\xx^\star)}^2 \right)}\\
	& \leq \frac{3}{n^2} \sum_{i = 1}^{n} \left( L^2 \norm{ \xx_i^{(t)} - \bar{\xx}}^2 + 2L \left(f_i(\bar{\xx}^{(t)}) - f_i(\xx^\star)\right) + \sigma_i^2 \right),
	\end{align*}
	where we used that $\E\norm{Y - a}^2 = \E\norm{Y}^2 - \norm{a}^2 \leq \E\norm{Y}^2$ if $a = \E Y$.
\end{proof}

\begin{lemma}[Descent lemma for convex cases]\label{lem:stochastic_avg}
	Under Assumptions~\ref{a:lsmooth}, \ref{a:strong}, \ref{a:opt} and \ref{a:avg_distrib},
	the averages $\bar{\xx}^{(t)} := \frac{1}{n}\sum_{i=1}^n \xx_i^{(t)}$ of the iterates of Algorithm \ref{alg:random_W_decentr_sgd_matrix} with the stepsize $\eta_t \leq \frac{1}{12 L} $ satisfy 
	\begin{align}
	\begin{split}
	\EE{\bxi_1^{(t)},\dots,\bxi_n^{(t)}}{\|\bar{\xx}^{(t + 1)} - \xx^\star\|}^2 &\leq \left(1 - \dfrac{\eta_t\mu}{2}\right) {\norm{\bar{\xx}^{(t)} - \xx^\star}}^2 + \dfrac{\eta_t^2\bar{\sigma}^2}{n} - \eta_t \left(f(\bar{\xx}^{(t)}) - f^\star\right) + \eta_t \dfrac{3 L }{n} \sum_{i = 1}^{n}\norm{\bar{\xx}^{(t)} - \xx_i^{(t)}}^2,
	\end{split}
	\end{align}
	where %
	$\bar{\sigma}^2 = \frac{1}{n}\sum_{i = 1}^{n} \sigma_i^2$. 
\end{lemma}
\begin{proof} Because all mixing matrixes preserve the average (Proposition~\ref{rem:average_preserve}), we have
	\begin{align*}
	\norm{\bar{\xx}^{(t + 1)} - \xx^\star}^2 
	&= \norm{\bar{\xx}^{(t)} - \frac{\eta_t}{n}\sum_{i = 1}^n  \nabla F_i(\xx_i^{(t)}, \xi_i^{(t)}) - \xx^\star}^2 \\
	&= \norm{\bar{\xx}^{(t)} - \xx^\star - \frac{\eta_t}{n} \sum_{i = 1}^{n}\nabla f_i(\xx_i^{(t)}) + \frac{\eta_t}{n} \sum_{i = 1}^{n}\nabla f_i(\xx_i^{(t)}) - \frac{\eta_t}{n}\sum_{i = 1}^n \nabla F_i(\xx_i^{(t)}, \xi_i^{(t)})}^2 \\
	&= \norm{\bar{\xx}^{(t)} - \xx^\star - \frac{\eta_t}{n} \sum_{i = 1}^{n}\nabla f_i(\xx_i^{(t)})}^2
	+\eta_t^2 \norm{\frac{1}{n} \sum_{i = 1}^{n}\nabla f_i(\xx_i^{(t)}) - \frac{1}{n}\sum_{i = 1}^n \nabla F_i(\xx_i^{(t)}, \xi_i^{(t)})}^2+\\
	&\qquad {}+ \frac{2\eta_t}{n}\left\langle  \bar{\xx}^{(t)} - \xx^\star - \frac{\eta_t}{n} \sum_{i = 1}^{n}\nabla f_i(\xx_i^{(t)}), \sum_{i = 1}^{n}\nabla f_i(\xx_i^{(t)}) - \sum_{i = 1}^n \nabla F_i(\xx_i^{(t)}, \xi_i^{(t)}) \right\rangle \,.
	\end{align*}
	The last term is zero in expectation, as $\EE{\xi_i^{(t)}}{\nabla F_i(\xx_i^{(t)}, \xi_i^{(t)}) } = \nabla f_i(\xx_i^{(t)})$. The second term is estimated using Proposition~\ref{rem:variance}.

	The first term can be written as:
	\begin{align*}
	\norm{\bar{\xx}^{(t)} - \xx^\star - \frac{\eta_t}{n} \sum_{i = 1}^{n}\nabla f_i(\xx_i^{(t)})}^2 = \norm{\bar{\xx}^{(t)} - \xx^\star}^2 + \eta_t^2 \underbrace{\norm{\frac{1}{n}\sum_{i = 1}^{n}\nabla f_i(\xx_i^{(t)})}^2}_{=: T_1} - \underbrace{2\eta_t \lin{\bar{\xx}^{(t)} - \xx^\star, \frac{1}{n}\sum_{i = 1}^{n}\nabla f_i(\xx_i^{(t)})}}_{=: T_2} \,.
	\end{align*}
	We can estimate 
	\begin{align*}
	T_1 &=
	\left\| \frac{1}{n}\sum_{i = 1}^{n}(\nabla f_i(\xx_i^{(t)}) -\nabla f_i(\bar{\xx}^{(t)}) + \nabla f_i(\bar{\xx}^{(t)}) -  \nabla f_i(\xx^\star))\right\|^2 \\
	&\stackrel{\eqref{eq:norm_of_sum}}{\leq}\frac{2}{n} \sum_{i = 1}^{n} \norm{\nabla f_i(\xx_i^{(t)}) - \nabla f_i(\bar{\xx}^{(t)})}^2  + 2\norm{\frac{1}{n} \sum_{i = 1}^{n} \nabla f_i(\bar{\xx}^{(t)}) - \frac{1}{n} \sum_{i = 1}^{n} \nabla f_i(\xx^\star)}^2 \\
	&\stackrel{\eqref{eq:f-lsmooth-convex},\eqref{eq:lsmooth_norm}}{\leq} \dfrac{2L^2}{n} \sum_{i = 1}^{n} \norm{\xx_i^{(t)} - \bar{\xx}^{(t)}}^2 + \dfrac{4L}{n}\sum_{i = 1}^{n} \left(f_i(\bar{\xx}^{(t)}) - f_i(\xx^\star)\right)
	\\& = \dfrac{2L^2}{n} \sum_{i = 1}^{n} \norm{\xx_i^{(t)} - \bar{\xx}^{(t)}}^2 + 4L \left(f(\bar{\xx}^{(t)}) - f^\star\right)\,.
	\end{align*}
	
	And for the remaining $T_2$ term:

	\begin{align*}
	- \frac{1}{\eta_t} T_2
	&= - \frac{2}{n} \sum_{i = 1}^{n}\left[\lin{\bar{\xx}^{(t)}  - \xx_i^{(t)}, \nabla f_i(\xx_i^{(t)})} + \lin{\xx_i^{(t)} - \xx^\star, \nabla f_i(\xx_i^{(t)})}\right]
	\\ &\stackrel{\eqref{eq:f-lsmooth},\eqref{eq:strongconv}}{\leq} - \dfrac{2}{n}\sum_{i = 1}^{n} \left[ f_i(\bar{\xx}^{(t)}) - f_i(\xx_i^{(t)}) - \dfrac{L}{2} \norm{\bar{\xx}^{(t)} - \xx_i^{(t)}}^2 + f_i(\xx_i^{(t)}) - f_i(\xx^\star) + \dfrac{\mu}{2} \norm{\xx_i^{(t)} - \xx^\star}^2\right] \\
	&{}\stackrel{\eqref{eq:norm_of_sum}}{\leq} - 2 \left(f(\bar{\xx}^{(t)}) - f(\xx^\star)\right) + \dfrac{L + \mu}{n}\sum_{i = 1}^{n} \norm{\bar{\xx}^{(t)} - \xx_i^{(t)}}^2 - \dfrac{\mu}{2} \norm{\bar{\xx}^{(t)} - \xx^\star}^2\,,
	\end{align*}
	Where at the last step \eqref{eq:norm_of_sum} was applied to $\norm{\bar{\xx}^{(t)} - \xx^\star}^2 \leq 2 \norm{\bar{\xx}^{(t)} - \xx_i^{(t)}}^2  + 2 \norm{{\xx}_i^{(t)} - \xx^\star}^2$.
	Putting everything together and using that $\eta_t \leq \frac{1}{12L}$ we are getting statement of the lemma.%
\end{proof}

\begin{lemma}[Recursion for consensus distance]\label{lem:consensus}
	Under Assumptions~\ref{a:lsmooth}, \ref{a:strong}, \ref{a:opt} and \ref{a:avg_distrib}, if in addition functions $F_i$ are convex and if stepsizes $\eta_t \leq  \frac{p}{96\sqrt{6}\tau  L}$, then 
	\begin{align*}
		\Xi_{t} & \leq \left( 1 - \frac{p}{2}\right) \Xi_{m\tau} + \frac{p}{64 \tau}  \sum_{j = m\tau}^{t - 1} \Xi_j + 72  \frac{\tau}{p} L\sum_{j = m\tau}^{t - 1} \eta_j^2 \left(f(\bar\xx^{(j)}) - f(\xx^\star)\right) + \left(8 \bar\sigma^2  + \frac{18\tau}{p} \bar\zeta^2\right) \sum_{j = m\tau}^{t - 1} \eta_j^2,
	\end{align*}
	where $\Xi_{t} = \frac{1}{n}\E{\sum_{i=1}^n  \norm{\xx_i^{(t)}-\bar \xx^{(t)}}^2}$ is a consensus distance, $m = \lfloor {t / \tau} \rfloor  - 1$. %
\end{lemma}
\begin{proof}
	Using matrix notation \eqref{eq:matrix_notation}, for $t \geq \tau$
	\begin{align*}
	n\Xi_{t} & = \E{\norm{X^{(t)} - \bar X^{(t)}}_F^2 } = \E \norm{X^{(t)}  - \bar X^{(m \tau)  }- \left(\bar X^{(t)} - \bar X^{(m \tau)  } \right)}_F^2 \leq \E \norm{X^{(t)}  - \bar X^{(m \tau)  }}_F^2,
	\end{align*}
	where we used that $\norm{A - \bar{A}}_F^2 = \sum_{i = 1}^n \norm{\aa_i - \bar \aa}_2^2 \leq \sum_{i = 1}^n \norm{\aa_i }_2^2 = \norm{A}_F^2$. Unrolling $X^{(t)}$ up to $X^{(m\tau)}$ using lines 3--4 of the Algorithm~\ref{alg:decentr_sgd_matrix},
	\begin{align*}
	n \Xi_{t}
	&\leq  \E \norm{X^{(m\tau)} \prod_{i = t - 1}^{m\tau} W^{(i)} - \bar X^{(m \tau)} - \sum_{j = m\tau}^{t - 1} \eta_j \partial F(X^{(j)}, \xi^{(j)})\prod_{i = t - 1}^{j} W^{(i)} }_F^2 \\
	&= \E \norm{X^{(m\tau)} \prod_{i = t - 1}^{m\tau} W^{(i)} - \bar X^{(m \tau)} - \sum_{j = m\tau}^{t - 1} \eta_j \partial f(X^{(j)}) \prod_{i = t - 1}^{j} W^{(i)} - \sum_{j = m\tau}^{t - 2} \eta_j \left(\partial F(X^{(j)}, \xi^{(j)}) - \partial f(X^{(j)})\right)\prod_{i = t - 1}^{j} W^{(i)} }_F^2 \\
	& \qquad \qquad + \E \norm{\eta_{t -1 } \left( \partial F(X^{(t - 1)}, \xi^{(t - 1)}) - \partial f(X^{(t - 1)}) \right) \prod_{i = t - 1}^{j} W^{(i)} }_F^2
	\end{align*}
	where we used that $\E \partial F(X^{(t - 1)}, \xi^{(t - 1)}) = \partial f(X^{(t - 1)})$ and that $\xi^{(t - 1)}$ is independent of the rest. To separate the rest of the stochastic terms similar way (terms with $\partial F(X^{(j)}, \xi^{(j)}) - \partial f(X^{(j)})$), since $X^{(t - 1)}$ depends on $\xi^{(t - 2)}$, we first need to separate the term with $\partial f(X^{(t - 1)})$. Let $\beta_1 = \frac{1}{C - 1}$ for some constant $C$ which we will define later,
	\begin{align*}
	n \Xi_{t} & \stackrel{\eqref{eq:norm_of_sum_of_two},\eqref{eq:frob_norm_of_multiplication}}{\leq} \underbrace{(1 + \beta_1)}_{= \frac{C}{C - 1}} \E \norm{X^{(m\tau)} \prod_{i = t - 1}^{m\tau} W^{(i)} - \bar X^{(m \tau)} - \left(\sum_{j = m\tau}^{t - 2} \eta_j \partial f(X^{(j)}) - \sum_{j = m\tau}^{t - 2} \eta_j \left(\partial F(X^{(j)}, \xi^{(j)}) - \partial f(X^{(j)})\right)\right)\prod_{i = t - 1}^{j} W^{(i)} }_F^2 \\
	& \qquad \qquad + \underbrace{(1 + \beta_1^{-1})}_{= C}\E \norm{\eta_{t -1 } \partial f(X^{(t - 1)})}_F^2 + \eta_{t -1 }^2\E \norm{  \partial F(X^{(t - 1)}, \xi^{(t - 1)}) - \partial f(X^{(t - 1)}) }_F^2,
	\end{align*}
	Now, similarly, we split terms that depend on $X^{(t - 2)}$ with $\beta_2 = \frac{1}{C - 2}$. Note that $(1 + \beta_1)(1 + \beta_2^{-1}) = C$ and $(1 + \beta_1)(1 + \beta_2) = \frac{C}{C - 2}$:
	\begin{align*}
	n \Xi_{t} &\leq \frac{C}{C - 2} \E \norm{X^{(m\tau)} \prod_{i = t - 1}^{m\tau} W^{(i)} - \bar X^{(m \tau)} - \left(\sum_{j = m\tau}^{t - 2} \eta_j \partial f(X^{(j)}) - \sum_{j = m\tau}^{t - 2} \eta_j \left(\partial F(X^{(j)}, \xi^{(j)}) - \partial f(X^{(j)})\right)\right)\prod_{i = t - 1}^{j} W^{(i)} }_F^2 \\
	& \qquad \qquad + C \sum_{j = t - 2}^{t - 1} \E \norm{\eta_{j} \partial f(X^{(j)})}_F^2 + \sum_{j = t - 2}^{t - 1} \frac{C}{C + j - (t - 1)} \eta_{j}^2\E \norm{  \partial F(X^{(j)}, \xi^{(j)}) - \partial f(X^{(j)}) }_F^2,
	\end{align*}
	Splitting the same way the rest of the terms and using that $\frac{C}{C + j - (t - 1)} \leq 2$ for $C \geq 2\tau$,
	\begin{align*}
	n \Xi_{t} & \leq \frac{C}{C - 2\tau} \E \norm{X^{(m\tau)} \prod_{i = t - 1}^{m\tau} W^{(i)}  - \bar X^{(m \tau)} }^2_F + C \sum_{j = m \tau}^{t - 1} \E \eta_{j}^2 \norm{\partial f(X^{(j)})}_F^2 + \sum_{j = m \tau}^{t - 1} 2 \eta_{j}^2\E \norm{  \partial F(X^{(j)}, \xi^{(j)}) - \partial f(X^{(j)}) }_F^2,
	\end{align*}
	Taking $C = 2 \tau (1 + \frac{2}{p})$ and using \eqref{eq:p} to bound the first term we get that
	\begin{align*}
	n \Xi_{t} & \leq \left(1 - \frac{p}{2}\right) \E \norm{X^{(m\tau)}  - \bar X^{(m \tau)} }^2_F + \frac{6\tau}{p} \sum_{j = m \tau}^{t - 1} \eta_{j}^2 \E \underbrace{\norm{\partial f(X^{(j)})}_F^2}_{:=T_1} + \sum_{j = m \tau}^{t - 1} 2 \eta_{j}^2 \underbrace{\E \norm{  \partial F(X^{(j)}, \xi^{(j)}) - \partial f(X^{(j)}) }_F^2}_{:=T_2},
	\end{align*}
	Estimating separately the last two terms, and using the notation $\pm a = a - a = 0 ~ \forall a$,
	\begin{align*}
	T_1 &= \E \norm{\partial f(X^{(j)}) \pm \partial f(\bar X^{(j)}) \pm \partial f(X^\star) }_F^2 \stackrel{\eqref{eq:norm_of_sum}}{\leq} 3 \E \norm{\partial f(X^{(j)}) - \partial f(\bar X^{(j)})}_F^2 + 3 \E\norm{\partial f(\bar X^{(j)}) - \partial f(X^\star)}_F^2 + 3 \norm{\partial f(X^\star)}_F^2 \\
	& \stackrel{\eqref{eq:F-smooth} , \eqref{eq:lsmooth_norm}, \eqref{eq:grad_opt} }{\leq} 3 \left( L^2 \E \norm{X^{(j)} -  \bar X^{(j)}}_F^2 + 2 L n \E\left( f(\bar{\xx}^{(j)}) - f(\xx^\star)\right) + n \bar\zeta^2 \right)
	\end{align*}
	\begin{align*}
	T_2 & = \E \norm{   \partial F(X^{(j)}, \xi^{(j)}) \pm \partial F(\bar X^{(j)}, \xi^{(j)}) \pm \partial F(X^\star, \xi^{(j)}) - \partial f(X^{(j)}) \pm  \partial f(\bar X^{(j)}) \pm \partial f(X^\star)}_F^2\\
	& \stackrel{\eqref{eq:norm_of_sum}, \eqref{eq:F-smooth}, \eqref{eq:lsmooth_norm}}{\leq} 4 \E \left(4 L^2 \norm{X^{(j)} -  \bar X^{(j)}}_F^2  + 4 L n \E\left( f(\bar{\xx}^{(j)}) - f(\xx^\star)\right) + \norm{\partial F(X^\star, \xi^{(j)}) - \partial f(X^\star)}_F^2 \right),
	\end{align*}
	where the last term is bounded by $n \bar \sigma^2$ by definition \eqref{eq:noise_opt}. Putting back estimates for $T_1$ and $T_2$ and using that $\eta_t \leq \frac{p}{96\sqrt{6}\tau  L}$ we arrive to the statement of the lemma.
\end{proof}

This recursion in Lemma~\ref{lem:consensus} holds only when $t \geq (m +1)\tau$. For these steps we are guaranteed to get $(1 - p)$ decrease by Assumption~\ref{a:avg_distrib}. To simplify this recursion we would need similar relation also for smaller $t$ that is $m\tau \leq t < (m + 1) \tau$, that we derive in Lemma~\ref{lem:consensus2}.

\begin{lemma}[Second recursion for consensus distance]\label{lem:consensus2}
	Under Assumptions~\ref{a:lsmooth}, \ref{a:strong}, \ref{a:opt} and \ref{a:avg_distrib}, if in addition functions $F_i$ are convex and if stepsizes $\eta_t \leq  \frac{p}{96\sqrt{6}\tau  L}$, then 
	\begin{align*}
	\Xi_{t} & \leq \left( 1 + \frac{p}{2}\right) \Xi_{m\tau} + \frac{p}{64 \tau}  \sum_{j = m\tau}^{t - 1} \Xi_j + 72 \frac{\tau}{p} L\sum_{j = m\tau}^{t - 1} \eta_j^2 \left(f(\bar\xx^{(j)}) - f(\xx^\star)\right) + \left(8 \bar\sigma^2  + \frac{18\tau}{p} \bar\zeta^2\right) \sum_{j = m\tau}^{t - 1} \eta_j^2,
	\end{align*}
	where $\Xi_{t} = \frac{1}{n}\E_t{\sum_{i=1}^n  \norm{\xx_i^{(t)}-\bar \xx^{(t)}}^2}$ is a consensus distance, and $t$ is such that $m\tau \leq t < (m + 1) \tau$. %
\end{lemma}
\begin{proof}
	The proof follows exactly the same lines as in Lemma~\ref{lem:consensus}, with the change that we don't use  \eqref{eq:p} to decrease the consensus distance by $(1 - p)$, but instead we use the Definition~\ref{def:valid_mixing} that each $W^{(i)}$ is doubly stochastic
\begin{align*}
\E \norm{X^{(m\tau)} \prod_{i = t - 1}^{m\tau} W^{(i)}  - \bar X^{(m \tau)} }^2_F \leq \E \norm{X^{(m\tau)} - \bar X^{(m \tau)} }^2_F.
\end{align*}	
\end{proof}

\subsection{Non-convex Case}
Here we derive descent recursive equation \eqref{eq:rec1} and recursion for consensus distance \eqref{eq:consensus_recursion} for the non-convex case.

\begin{proposition}[Mini-batch variance]\label{rem:variance_nc} Let functions $F_i(\xx, \xi)$ , $i \in [n]$ be $L$-smooth (Assumption~\ref{a:lsmooth}) with bounded noise as in Assumption~\ref{a:opt_nc}. Then for any $\xx_i \in \R^d, i \in [n]$ and $\bar \xx := \frac{1}{n}\sum_{i=1}^n \xx_i$ it holds
	\begin{align}
	\EE{\xi_1, \dots, \xi_n}{\norm{\frac{1}{n}\sum_{i = 1}^{n} \left(\nabla f_i(\xx_i) - \nabla F_i (\xx_i, \xi_i)\right)}}^2\leq \frac{\hat \sigma^2}{n}  + \frac{M}{n^2} \sum_{i = 1}^n \norm{\nabla f(\xx_i)}^2 \label{eq:mini-batch}
	\end{align}
\end{proposition}

\begin{lemma}[Descent lemma for non-convex case]\label{lem:decrease_nc} Under Assumptions~\ref{a:lsmooth_nc}, \ref{a:opt_nc} and \ref{a:avg_distrib},
	the averages $\bar{\xx}^{(t)} := \frac{1}{n}\sum_{i=1}^n \xx_i^{(t)}$ of the iterates of Algorithm \ref{alg:random_W_decentr_sgd_matrix} with the constant stepsize $\eta < \frac{1}{4 L (M + 1)}$ satisfy 
	\begin{align}\label{eq:main_recursion_nc}
	\EE{t + 1}{f(\bar{\xx}^{(t + 1)})} &\leq f(\bar{\xx}^{(t)}) - \frac{\eta}{4}\norm{\nabla f(\bar{\xx}^{(t)})}_2^2 + \frac{\eta L^2}{n} \sum_{i = 1}^n \norm{\bar{\xx}^{(t)} -  \xx_i^{(t)}}_2^2 + \frac{L}{n} \eta^2 \hat{\sigma}^2.
	\end{align}
\end{lemma}

\begin{proof}
	Because all mixing matrixes preserve the average (Proposition~\ref{rem:average_preserve}) and function $f$ is $L$-smooth, we have 
	\begin{align*}
	\EE{t + 1}{f(\bar{\xx}^{(t + 1)})} &= \EE{t + 1}{f\left(\bar{\xx}^{(t)} - \frac{\eta}{n}\sum_{i = 1}^n \nabla F_i(\xx_i^{(t)}, \xi_i^{(t)})\right) }\\
	& \leq f(\bar{\xx}^{(t)}) - \underbrace{\EE{t + 1}{\lin{ \nabla f(\bar{\xx}^{(t)}),  \frac{\eta}{n} \sum_{i = 1}^n \nabla F_i(\xx_i^{(t)}, \xi_i^{(t)}) }} }_{:=T_1}+ \EE{t + 1}{\frac{L}{2} \eta^2 \underbrace{ \norm{\frac{1}{n} \sum_{j = 1}^n \nabla F_i(\xx_i^{(t)}, \xi_i^{(t)})}_2^2}_{:=T_2} }\\
	\end{align*}
	To estimate the second term, we add and subtract $\nabla f(\bar{\xx}^{(t)})$ 
	\begin{align*}
	T_1 &= - \eta \norm{\nabla f(\bar{\xx}^{(t)})}^2 +  \frac{\eta}{n} \sum_{i = 1}^n  \lin{ \nabla f(\bar{\xx}^{(t)}),  \nabla f_i(\bar{\xx}^{(t)}) - \nabla f_i(\xx_i^{(t)}) } \\ & \stackrel{\eqref{eq:scal_product}, \gamma = 1; \eqref{eq:norm_of_sum}}{\leq} - \frac{\eta}{2}\norm{\nabla f(\bar{\xx}^{(t)})}^2 + \dfrac{\eta}{2 n} \sum_{i = 1}^n \norm{\nabla f_i(\bar{\xx}^{(t)}) - \nabla f_i(\xx_i^{(t)})}^2
	\end{align*}
	For the last term, using the notation $~\pm a = a - a = 0~~\forall a$,
	\begin{align*}
	T_2 & =  \E_{t + 1}\norm{\frac{1}{n} \sum_{j = 1}^n \left(\nabla F_i(\xx_i^{(t)}, \xi_i^{(t)})-\nabla f_i(\xx_i^{(t)}) \right)}_2^2 + \norm{\frac{1}{n}\sum_{i = 1}^n\nabla f_i(\xx_i^{(t)})}^2_2 \\
	& \stackrel{\eqref{eq:mini-batch}}{\leq }  \frac{\hat \sigma^2}{n}  + \frac{M}{n^2} \sum_{i = 1}^n \norm{\nabla f(\xx_i^{(t)}) \pm \nabla f(\bar\xx^{(t)}) }^2 + \norm{\frac{1}{n}\sum_{i = 1}^n\nabla f_i(\xx_i^{(t)}) \pm \nabla f(\bar\xx^{(t)})}^2_2\\
	& \stackrel{\eqref{eq:norm_of_sum_of_two}}{\leq} \frac{\hat \sigma^2}{n}  + \frac{2 M}{n^2} \sum_{i = 1}^n \norm{\nabla f(\xx_i^{(t)}) - \nabla f(\bar\xx^{(t)}) }^2 + \left(2 \nicefrac{M}{n} + 2 \right) \norm{\nabla f(\bar\xx^{(t)})}_2^2 + \frac{2}{n}\sum_{i = 1}^n\norm{\nabla f_i(\xx_i^{(t)}) - \nabla f_i(\bar\xx^{(t)})}^2_2
	\end{align*}
	Combining this together and using $L$-smoothness to estimate $\norm{\nabla f_i(\bar{\xx}^{(t)}) - \nabla f_i(\xx_i^{(t)})}_2^2$ and $\norm{\nabla f(\bar{\xx}^{(t)}) - \nabla f(\xx_i^{(t)})}_2^2$,
	\begin{align*}
	\EE{t + 1}{f(\bar{\xx}^{(t + 1)})} &\leq f(\bar{\xx}^{(t)}) - \eta\left(\frac{1}{2} - L\eta (M + 1)\right)\norm{\nabla f(\bar{\xx}^{(t)})}_2^2 + \left(\frac{\eta L^2}{2n}  + \frac{L^3 \eta^2 (M + 1)}{n}\right)\sum_{i = 1}^n \norm{\bar{\xx}^{(t)} -  \xx_i^{(t)}}_2^2 + \frac{L}{n} \eta^2 \hat{\sigma}^2 .
	\end{align*}
	Applying $\eta < \frac{1}{4 L (M + 1)}$ we get statement of the lemma. 	
\end{proof}

\begin{lemma}[Recursion for consensus distance]\label{lem:consensus_nc}
	Under Assumptions~\ref{a:lsmooth_nc}, \ref{a:opt_nc} and \ref{a:avg_distrib}, if the stepsize $\eta_t \leq  \frac{p}{8 L \sqrt{2\tau (6\tau + pM ) } }$, then 
	\begin{align*}
	\Xi_{t} & \leq \left( 1 - \frac{p}{2}\right) \Xi_{m\tau} + \frac{p}{16 \tau}  \sum_{j = m\tau}^{t - 1} \Xi_j + 2P \left(\frac{6\tau}{p} + M \right) \sum_{j = m\tau}^{t - 1} \eta_j^2 \norm{ \nabla f(\bar \xx^{(j)}) }_2^2 + \left(2 \hat \sigma^2  + 2\left(\frac{6 \tau}{p} + M \right)\hat \zeta^2 \right) \sum_{j = m\tau}^{t - 1} \eta_j^2
	\end{align*}
	where $\Xi_{t} = \frac{1}{n}\E_t{\sum_{i=1}^n  \norm{\xx_i^{(t)}-\bar \xx^{(t)}}^2}$ is a consensus distance, $m = \lfloor {t / \tau} \rfloor  - 1$. %
\end{lemma}
\begin{proof}
	We start exactly the same way as in the convex proof in Lemma~\ref{lem:consensus}
	Defining $\Xi_{t} = \frac{1}{n}\E_t{\sum_{i=1}^n  \norm{\xx_i^{(t)}-\bar \xx^{(t)}}^2}$, $m = \lfloor {t / \tau} \rfloor  - 1$ and using matrix notation \eqref{eq:notation_sgd}, for $t \geq \tau$ (and therefore $m \geq 0$)
	\begin{align*}
	n\Xi_{t} & = \E{\norm{X^{(t)} - \bar X^{(t)}}_F^2 } = \E \norm{X^{(t)}  - \bar X^{(m \tau )  }- \left(\bar X^{(t)} - \bar X^{(m \tau )  } \right)}_F^2 \leq \E \norm{X^{(t)}  - \bar X^{(m \tau)  }}_F^2,
	\end{align*}
	where we used that $\norm{A - \bar{A}}_F^2 = \sum_{i = 1}^n \norm{\aa_i - \bar \aa} \leq \sum_{i = 1}^n \norm{\aa_i }_F^2 = \norm{A}_F^2$. Unrolling $X^{(t)}$ up to $X^{(m\tau)}$ using lines 3-4 of the Algorithm~\ref{alg:decentr_sgd_matrix} and splitting stochastic terms similar way as for the convex cases in Lemma~\ref{lem:consensus},
	\begin{align*}
	n \Xi_{t} 
	&\leq  \E \norm{X^{(m\tau)} \prod_{i = t - 1}^{m\tau} W^{(i)} - \bar X^{(m \tau )} - \sum_{j = m\tau}^{t - 1} \eta_j \partial F(X^{(j)}, \xi^{(j)})\prod_{i = t - 1}^{j} W^{(i)} }_F^2 \\
	& \leq \left(1 - \frac{p}{2}\right) \E \norm{X^{(m\tau)}  - \bar X^{(m \tau)} }^2_F + \frac{6\tau}{p} \sum_{j = m \tau}^{t - 1} \eta_{j}^2 \E \norm{\partial f(X^{(j)})}_F^2 + \sum_{j = m \tau}^{t - 1} 2 \eta_{j}^2 \E \norm{  \partial F(X^{(j)}, \xi^{(j)}) - \partial f(X^{(j)}) }_F^2\\
	& \stackrel{\eqref{eq:noise_opt_nc} }{\leq } \left(1 - \frac{p}{2}\right) \E\norm{X^{(m\tau)}  - \bar X^{(m \tau)} }_F^2 + \left(\frac{6 \tau}{p} + M \right)\sum_{j = m\tau}^{t - 1}\eta_j^2 \underbrace{\norm{\partial f(X^{(j)}) }_F^2}_{:=T} + \sum_{j = m\tau}^{t - 1} 2 \eta_j^2 n \hat \sigma^2 \\
	\end{align*}
	Estimating $T$,
	\begin{align*}
	T &\stackrel{\eqref{eq:norm_of_sum_of_two} }{\leq } 2 \norm{\partial f(X^{(j)}) - \partial f(\bar X^{(j)}) }_F^2 + 2 \norm{\partial f(\bar X^{(j)})}_F^2 \stackrel{\eqref{eq:smooth_nc}, \eqref{eq:grad_opt_nc}}{\leq} 2 L^2 \norm{X^{(j)} - \bar X^{(j)}}_F^2  + 2 n \hat \zeta^2 + 2 P n \norm{\nabla f(\bar \xx^{(j)})}_2^2 
	\end{align*}
	Putting back estimate for $T$ and using that $\eta_t \leq \frac{p}{8 L \sqrt{2\tau (6\tau + pM ) } }$ we arrive to the statement of this lemma. 	
\end{proof}
Similarly to the convex cases, we additionally need a recursion for values $t$ that are in between $m\tau \leq t < (m + 1) \tau$ 
\begin{lemma}[Second recursion for consensus distance]\label{lem:consensus_nc2}
	Under Assumptions~\ref{a:lsmooth_nc}, \ref{a:opt_nc} and \ref{a:avg_distrib}, if the stepsize $\eta_t \leq  \frac{p}{8 L \sqrt{2\tau (6\tau + pM ) } }$, and $t$ such that $m\tau \leq t < (m + 1) \tau$ then 
	\begin{align*}
	\Xi_{t} & \leq \left( 1 + \frac{p}{2}\right) \Xi_{m\tau} + \frac{p}{64 \tau}  \sum_{j = m\tau}^{t - 1} \Xi_j + 2P \left(\frac{6\tau}{p} + M \right) \sum_{j = m\tau}^{t - 1} \eta_j^2 \norm{ \nabla f(\bar \xx^{(j)}) }_2^2 + \left(\hat \sigma^2  + 2\left(\frac{6 \tau}{p} + M \right)\hat \zeta^2 \right) \sum_{j = m\tau}^{t - 1} \eta_j^2
	\end{align*}
	where $\Xi_{t} = \frac{1}{n}\E_t{\sum_{i=1}^n  \norm{\xx_i^{(t)}-\bar \xx^{(t)}}^2}$ is a consensus distance. %
\end{lemma}
\begin{proof}
	As in the convex case, we need to change the proof of Lemma~\ref{lem:consensus_nc} just slightly, by applying Def.~\ref{def:valid_mixing} instead of \eqref{eq:p} as follows
	\begin{align*}
	\E \norm{X^{(m\tau)} \prod_{i = t - 1}^{m\tau} W^{(i)}  - \bar X^{(m \tau)} }^2_F \leq \E \norm{X^{(m\tau)} - \bar X^{(m \tau)} }^2_F
	\end{align*}
\end{proof}

\subsection{Simplifying Consensus Recursion}
In Lemmas~\ref{lem:consensus}, \ref{lem:consensus_nc} we obtained the consensus recursive equation \eqref{eq:consensus_recursion} for both convex and non-convex cases. In this section we simplify it to be able to easily combine it later with \eqref{eq:rec1}.

\begin{lemma}\label{lem:solve_consensus_recursion}
	If non-negative sequences $\{\Xi_t\}_{t \geq 0}$, $ \{e_t\}_{t \geq 0}$ and $\{\eta_t\}_{t\geq 0}$ satisfy \eqref{eq:consensus_recursion} and \eqref{eq:consensus_recursion2} for some constants $0 < p \leq 1, \tau \geq 1, A, D \geq 0$,  moreover if the stepsizes $\{\eta_t^2\}_{t\geq 0}$ is $\frac{8\tau}{p}$-slow decreasing sequence (Definition~\ref{def:tau-slow}), and if $\{w_t\}_{t \geq 0}$ is $\frac{16\tau}{p}$-slow increasing non-negative sequence of weights, then it holds that 
	\begin{align*}
	B \sum_{t = 0}^T w_t \Xi_t  \leq \frac{b}{2}\sum_{t= 0}^T w_t e_t  +64 B A \frac{\tau}{p} \sum_{t = 0}^Tw_t \eta_t^2,
	\end{align*}
	for some constant $B > 0$ with the constraint that stepsizes $\eta_t \leq \frac{1}{16}\sqrt{\frac{p b}{D B \tau}}$.
\end{lemma}
\begin{proof}
	Recursively substituting every $\Xi_j$ for $j \geq (m + 1)\tau$ in the second term of \eqref{eq:consensus_recursion} we get
	\begin{align*}
		\Xi_{t} & \leq \left( 1 - \frac{p}{2}\right) \Xi_{m\tau} \left(1 + \frac{p}{64\tau}\right)^{\tau}  + \left(1 + \frac{p}{64\tau}\right)^{\tau}  \frac{p}{64 \tau} \sum_{j = m\tau}^{(m + 1)\tau - 1} \Xi_j + D \sum_{j = (m + 1)\tau}^{t - 1} \left(1 + \frac{p}{64\tau}\right)^{t - 1- j} \eta_j^2 e_j \\
		& \qquad \qquad + D \sum_{j = m \tau}^{(m + 1) \tau - 1} \left(1 + \frac{p}{64\tau}\right)^{t - (m +1)\tau} \eta_j^2 e_j  +  A \sum_{j = (m + 1)\tau}^{t - 1} \left(1 + \frac{p}{64\tau}\right)^{t - 1- j} \eta_j^2 + A \sum_{j = m\tau}^{(m + 1) \tau - 1} \left(1 + \frac{p}{64\tau}\right)^{t - (m +1)\tau} \eta_j^2 
	\end{align*}
	We substitute the rest of $\Xi_j$ for $m \tau \leq j < (m + 1)\tau$ with \eqref{eq:consensus_recursion2}. Lets start with substituting $\Xi_{(m +1)\tau - 1 }$ 
	\begin{align*}
		\Xi_{t} & \leq \left(1 + \frac{p}{64\tau}\right)^{\tau}\left[\left( 1 - \frac{p}{2}\right) \Xi_{m\tau}  +\frac{p}{64 \tau} \left( 1 + \frac{p}{2}\right) \Xi_{m\tau} +  \left(1 + \frac{p}{64\tau}\right) \frac{p}{64 \tau} \sum_{j = m\tau}^{(m + 1)\tau - 2} \Xi_j \right]+  \\
		& \qquad \qquad + D \sum_{j = (m + 1)\tau - 1}^{t - 1} \left(1 + \frac{p}{64\tau}\right)^{t - 1- j} \eta_j^2 e_j + D \sum_{j = m \tau}^{(m + 1) \tau - 2} \left(1 + \frac{p}{64\tau}\right)^{t - (m +1)\tau + 1} \eta_j^2 e_j   \\
		& \qquad \qquad +  A \sum_{j = (m + 1)\tau  - 1}^{t - 1} \left(1 + \frac{p}{64\tau}\right)^{t - 1- j} \eta_j^2 + A \sum_{j = m\tau}^{(m + 1) \tau - 2} \left(1 + \frac{p}{64\tau}\right)^{t - (m +1)\tau + 1} \eta_j^2 
	\end{align*}
	Since $0 < p \leq 1$, it holds that $\frac{p}{64 \tau} \left( 1 + \frac{p}{2}\right) \leq \left(1 - \frac{p}{2}\right)\frac{p}{16 \tau}$ and therefore
	\begin{align*}
	\Xi_{t} & \leq \left(1 + \frac{p}{64\tau}\right)^{\tau}\left[\left( 1 - \frac{p}{2}\right) \Xi_{m\tau} \left(1 + \frac{p}{16 \tau}\right) +  \left(1 + \frac{p}{64\tau}\right) \frac{p}{64 \tau} \sum_{j = m\tau}^{(m + 1)\tau - 2} \Xi_j \right]+  \\
	& \qquad \qquad + D \sum_{j = (m + 1)\tau - 1}^{t - 1} \left(1 + \frac{p}{64\tau}\right)^{t - 1- j} \eta_j^2 e_j + D \sum_{j = m \tau}^{(m + 1) \tau - 2} \left(1 + \frac{p}{64\tau}\right)^{t - (m +1)\tau + 1} \eta_j^2 e_j   \\
	& \qquad \qquad +  A \sum_{j = (m + 1)\tau  - 1}^{t - 1} \left(1 + \frac{p}{64\tau}\right)^{t - 1- j} \eta_j^2 + A \sum_{j = m\tau}^{(m + 1) \tau - 2} \left(1 + \frac{p}{64\tau}\right)^{t - (m +1)\tau + 1} \eta_j^2 
	\end{align*}
	Applying the same way \eqref{eq:consensus_recursion2} to the rest of $\Xi_j$ and using that $\frac{p}{64\tau}\leq \frac{p}{16\tau}$ we get that 
	\begin{align*}
		\Xi_{t} & \leq \left( 1 - \frac{p}{2}\right) \Xi_{m\tau} \left(1 + \frac{p}{16\tau}\right)^{2\tau}+ D \sum_{j = m\tau }^{t - 1} \left(1 + \frac{p}{16\tau}\right)^{t - 1- j} \eta_j^2 e_j +  A \sum_{j = m\tau}^{t - 1} \left(1 + \frac{p}{16\tau}\right)^{t - 1- j} \eta_j^2
	\end{align*}
	
	Using that $ \left(1 + \frac{p}{16\tau}\right)^{2\tau} \leq \exp\left( \frac{p}{8} \right) \leq 1 + \frac{p}{4}$ for $p \leq 1$ and also  that $(1 + \frac{p}{16\tau})^{t - 1 - j} \leq \left(1 + \frac{p}{16\tau}\right)^{2\tau} \leq 1 + \frac{p}{4} \leq 2 $
	\begin{align*}
	\Xi_{t} & \leq \left( 1 - \frac{p}{4}\right) \Xi_{m\tau} + 2 D \sum_{j = m\tau}^{t - 1} \eta_j^2 e_j + 2 A \sum_{j = m\tau}^{t - 1} \eta_j^2,
	\end{align*}
	Unrolling $\Xi_{m\tau}$ recursively up to $0$ we get, 
	\begin{align*}
	\Xi_{t} & \leq 2D \sum_{j = 0}^{t - 1} \left( 1 - \frac{p}{4}\right)^{\lfloor (t  - j ) / \tau \rfloor} \eta_j^2 e_j + 2 A \sum_{j = 0}^{t - 1} \left( 1 - \frac{p}{4}\right)^{\lfloor (t - j ) / \tau \rfloor} \eta_j^2,
	\end{align*}
	For the first term estimating $\left( 1 - \frac{p}{4}\right)^{1/\tau} \leq  \exp(-\frac{p}{4\tau}) \leq 1- \frac{p}{8\tau} $ and that $\left( 1 - \frac{p}{8\tau} \right)^{\tau \lfloor ( t - j ) / \tau \rfloor} \leq \left( 1 - \frac{p}{8\tau} \right)^{ t - j }\left( 1 - \frac{p}{8\tau} \right)^{ -\tau }$. For the last term, $\left( 1 - \frac{p}{8\tau} \right)^{ -\tau } \leq \left(\frac{1}{1 - \frac{p}{8\tau}}\right)^\tau \leq ( 1 + \frac{p}{4\tau})^\tau$ because $\frac{p}{8\tau} \leq \frac{1}{2}$ and finally $\left( 1 + \frac{p}{4\tau}\right)^\tau \leq \exp(\frac{p}{4})  < 2$,
	\begin{align*}
	\Xi_{t} & \leq 4 D \sum_{j = 0}^{t - 1} \left( 1 - \frac{p}{8 \tau}\right)^{t  - j } \eta_j^2 e_j + 4 A \sum_{j = 0}^{t - 1} \left( 1 - \frac{p}{8 \tau }\right)^{ t - j } \eta_j^2,
	\end{align*}
	Now using that $\eta_t^2$ is $\frac{8\tau}{p}$-slow decreasing, i.e. $\eta_j^2 \leq \eta_t^2 \left(1 + \frac{p}{16\tau}\right)^{t - j}$ and using that $(1 - \frac{p}{8\tau})(1 + \frac{p}{16\tau}) \leq (1 - \frac{p}{16\tau})$ 
	\begin{align*}
	\Xi_{t} & \leq 4 D \eta_t^2 \sum_{j = 0}^{t - 1} \left( 1 - \frac{p}{16 \tau}\right)^{t  - j } e_j + 4 A \eta_t^2 \sum_{j = 0}^{t - 1} \left( 1 - \frac{p}{16 \tau }\right)^{ t - j } \leq 4 D \eta_t^2 \sum_{j = 0}^{t - 1} \left( 1 - \frac{p}{16 \tau}\right)^{t  - j } e_j + 64 A \frac{\tau}{p}  \eta_t^2
	\end{align*}
	Now averaging $\Xi_{t}$ with weights $w_t$ and using that $w_t$ is $\frac{16\tau}{p}$-slow increasing sequence, i.e. $w_t \leq w_j \left( 1 + \frac{p}{32\tau}\right)^{t - j}$, and also using that $\eta_t \leq \frac{1}{16}\sqrt{\frac{p b}{DB\tau}}$
	\begin{align*}
	B \sum_{t = 0}^T w_t \Xi_{t} & \leq 4 D B  \sum_{t = 0}^T \eta_t^2 \sum_{j = 0}^{t - 1} w_j \left( 1 - \frac{p}{32 \tau}\right)^{t  - j } e_j + 64 A B \frac{\tau}{p} \sum_{t = 0}^T w_t \eta_t^2  \\ &\leq \underbrace{\frac{p b}{64 \tau}  \sum_{t = 0}^T \sum_{j = 0}^{t - 1} w_j \left( 1 - \frac{p}{32 \tau}\right)^{t  - j } e_j}_{:= T_1} + 64 A B \frac{\tau}{p} \sum_{t = 0}^T w_t \eta_t^2 
	\end{align*}
	And finally,
	\begin{align*}
	T_1 &= \frac{p b}{64 \tau} \sum_{j = 0}^T w_j e_j\sum_{t = j + 1}^{T}  \left( 1 - \frac{p}{32 \tau}\right)^{t  - j } \leq \frac{p b}{64 \tau} \sum_{j = 0}^T w_j e_j\sum_{t = 0}^{\infty}  \left( 1 - \frac{p}{32 \tau}\right)^{t  - j } \leq \frac{b}{2} \sum_{t = 0}^T w_t e_t. \qedhere
	\end{align*}
\end{proof}

\section{Solving the Main Recursion \eqref{eq:rec3}}
\label{sec:mainrecursionlemmas}
\subsection{$a > 0$ (strongly convex case)}

\begin{lemma}\label{lem:rate_strongly_convex}
	If non-negative sequences $\{r_t\}_{t\geq 0}, \{e_t\}_{t \geq 0}$ satisfy \eqref{eq:rec3} for some constants $a, b, p~>~0, c, A, B, \tau~\geq~0$, then there exists a constant stepsize $\eta_t = \eta < \frac{1}{d}$ such that for weights $w_t = (1 - a \eta)^{-(t + 1)}$ and $W_T := \sum_{t= 0}^T w_t$ it holds:
	\begin{align*}
	\frac{1}{2 W_T} \sum_{t= 0}^T b e_t w_t + a r_{T+1} \leq \tilde{\cO}\left(r_0 d\exp\left[-\frac{a (T + 1)}{d}\right] + \frac{c }{aT} + \frac{B A}{a^2T^2} \frac{\tau}{p} \right),
	\end{align*}
	where $\tilde{\cO}$ hides polylogarithmic factors. %
\end{lemma}
\begin{proof}
	Starting from \eqref{eq:rec3} and using that $\eta_t = \eta$ and that $\frac{w_t (1 - a \eta)}{\eta} = \frac{w_{t - 1}}{\eta}$ we obtain a telescoping sum,
	\begin{align*}
	\frac{1}{2 W_T}\sum_{t = 0}^T b w_t e_t \leq \frac{1}{W_T\eta} \left((1 - a \eta)w_0 r_0 - w_T  r_{T + 1} \right)  + c \eta + 64B A \frac{\tau}{p}\eta^2\,,
	\end{align*}
	And hence, 
	\begin{align*}
	\frac{1}{2 W_T}\sum_{t = 0}^T b w_t e_t + \frac{w_T  r_{T + 1}}{W_T\eta} \leq \frac{ r_0 }{W_T\eta} + c \eta + 64B A \frac{\tau}{p}\eta^2\,,
	\end{align*}
	Using that $W_T \leq \frac{w_T}{a \eta} $ and $W_T \geq w_T = (1 - a \gamma)^{-(T + 1)}$ we can simplify
	\begin{align*}
	\frac{1}{2 W_T}\sum_{t = 0}^T b w_t e_t + a  r_{T + 1}\leq (1 - a\eta)^{T + 1} \frac{r_0 }{\eta} + c \eta + 64B A \frac{\tau}{p}\eta^2 \leq \frac{r_0}{\eta} \exp\left[-a\eta (T + 1)\right] + c \eta + 64B A\frac{\tau}{p} \eta^2\,,
	\end{align*}
	Now lemma follows by tuning $\eta$ the same way as in \cite{Stich19sgd}.
	\begin{itemize}
		\item  If $\frac{1}{d} \geq \frac{\ln(\max\{2, a^2 r_0 T^2/c\})    }{aT}$  then we choose $\eta = \frac{\ln(\max\{2, a^2 r_0 T^2/c\})    }{aT}$ and get that 
		\begin{align*}
		\tilde{\cO}\left(a r_0 T\exp\left[-\ln(\max\{2, a^2 r_0 T^2/c\})  \right] \right) + \tilde{\cO}\left(\frac{c }{aT} \right) + \tilde{\cO}\left(\frac{B A}{a^2 T^2} \frac{\tau}{p} \right) = \tilde{\cO}\left(\frac{c }{aT} \right) + \tilde{\cO}\left(\frac{B A}{a^2 T^2} \frac{\tau}{p}  \right) \,,
		\end{align*}
		\item Otherwise $ \frac{1}{d} \leq \frac{\ln(\max\{2, a^2 r_0 T^2/c\})    }{aT}$ we pick $\eta = \frac{1}{d}$ and get that 
		\begin{align*}		
		&\tilde \cO\left(r_0 d\exp\left[-\frac{a (T + 1)}{d}\right] + \frac{c }{d} + \frac{B A}{d^2} \frac{\tau}{p} \right) \leq \tilde{\cO}\left(r_0 d\exp\left[-\frac{a (T + 1)}{d}\right] + \frac{c }{aT} + \frac{B A}{a^2T^2} \frac{\tau}{p}\right)\,. \qedhere
		\end{align*}
	\end{itemize}
\end{proof}

\subsection{$a = 0$ (weakly convex and non-convex cases)}
Now we assume that in Assumption~\ref{a:strong} $\mu = 0$, which means that $a = 0$ in \eqref{eq:rec3}.
\begin{lemma}\label{lem:rate_weakly_convex}
	If non-negative sequences $\{r_t\}_{t\geq 0}, \{e_t\}_{t \geq 0}$ satisfy \eqref{eq:rec3} with $a = 0, b~>~0, c, A, B~\geq~0$, then there exists a constant stepsize $\eta_t = \eta < \frac{1}{d}$ such that for weights $\{w_t = 1\}_{t \geq 0}$ it holds that:
	\begin{align*}
	\frac{1}{(T + 1)} \sum_{t= 0}^T e_t \leq \cO \left(2  \left(\frac{c r_0}{T + 1}\right)^{\frac{1}{2}} +  2 \left(\frac{B A \tau}{p}\right)^{1/3}\left(\frac{r_0}{T + 1}\right)^{\frac{2}{3}} + \frac{d r_0}{T + 1} \right).
	\end{align*}
\end{lemma}
\begin{proof}
	With $a = 0$, constant stepsizes $\eta_t = \eta$ and weights $\{w_t = 1\}_{t \geq 0}$ \eqref{eq:rec3} is equivalent to 
	\begin{align*}
	\frac{1}{2(T + 1) }\sum_{t = 0}^T e_t \leq \frac{1}{(T + 1)\eta}\sum_{t=0}^T \left( r_t - r_{t + 1} \right)  + c \eta + 64\frac{B A \tau}{p} \eta^2 \leq \frac{r_0}{(T + 1)\eta} + c \eta + 64 \frac{B A \tau}{p} \eta^2.
	\end{align*}
	To conclude the proof we tune the stepsize using Lemma~\ref{lem:tuning_stepsize}.
\end{proof}
\begin{lemma}[Tuning the stepsize]\label{lem:tuning_stepsize}
	For any parameters $r_0 \geq 0, b \geq 0, e \geq 0, d \geq 0$ there exists constant stepsize $\eta \leq \frac{1}{d}$ such that 
	\begin{align*}
	\Psi_T :=  \frac{r_0}{\eta(T + 1)} + b \eta  + e \eta^2 \leq 2  \left(\frac{b r_0}{T + 1}\right)^{\frac{1}{2}} +  2 e^{1/3}\left(\frac{r_0}{T + 1}\right)^{\frac{2}{3}} + \frac{d r_0}{T + 1}
	\end{align*}
\end{lemma}
\begin{proof}
	Choosing $\eta = \min\left\{\left(\frac{r_0}{b(T + 1)}\right)^{\frac{1}{2}}, \left(\frac{r_0}{e(T + 1)}\right)^{\frac{1}{3}} , \frac{1}{d} \right\} \leq \frac{1}{d}$ we have three cases
	\begin{itemize}
		\item $\eta  = \frac{1}{d}$ and is smaller than both $\left(\frac{r_0}{b(T + 1)}\right)^{\frac{1}{2}}$ and $\left(\frac{r_0}{e(T + 1)}\right)^{\frac{1}{3}} $, then 
		\begin{align*}
		\Psi_T &\leq  \frac{d r_0}{T + 1} + \frac{b}{d} + \frac{e}{d^2} \leq \left(\frac{b r_0}{T + 1}\right)^{\frac{1}{2}} + \frac{d r_0}{T + 1} + e^{1/3}\left(\frac{r_0}{T + 1}\right)^{\frac{2}{3}}
		\end{align*}
		\item $\eta = \left(\frac{r_0}{b(T + 1)}\right)^{\frac{1}{2}} < \left(\frac{r_0}{e(T + 1)}\right)^{\frac{1}{3}} $, then
		\begin{align*}
		\Psi_T &\leq  2 \left(\frac{r_0b}{T + 1 }\right)^{\frac{1}{2}}   + e  \left(\frac{r_0}{b(T + 1)}\right) \leq   2 \left(\frac{r_0b}{T + 1}\right)^{\frac{1}{2}} + e^{\frac{1}{3}} \left(\frac{r_0}{(T + 1)}\right)^{\frac{2}{3}},
		\end{align*}
		\item The last case, $\eta = \left(\frac{r_0}{e(T + 1)}\right)^{\frac{1}{3}} < \left(\frac{r_0}{b(T + 1)}\right)^{\frac{1}{2}} $
		\begin{align*}
		\Psi_T &\leq  2 e^{\frac{1}{3}} \left(\frac{r_0}{(T + 1)}\right)^{\frac{2}{3}} + b \left(\frac{r_0}{e(T + 1)}\right)^{\frac{1}{3}} \leq 2 e^{\frac{1}{3}} \left(\frac{r_0}{(T + 1)}\right)^{\frac{2}{3}} + \left(\frac{b r_0}{T + 1}\right)^{\frac{1}{2}} \,. \qedhere
		\end{align*}
	\end{itemize}
\end{proof}

\section{Lower Bound}
\begin{proof}[Proof of Theorem~\ref{thm:lower_bound}]
	
	We consider minimization problem of the form~\eqref{eq:f} with $f_i(x) = \frac{1}{2} (x - y_i)^2$, $x, y_i \in \R$ which has the solution $x^\star = \frac{1}{n} \sum_{i=1}^n y_i$, $L = \mu = 1$. We denote $\xx = (x_1, \dots, x_n)^\top$ and $\nabla f(\xx) = \left(\nabla f_1(x_1), \dots, \nabla f_n(x_n) \right)^\top$. 
	
	We assume that the starting point $\xx^{(0)}$ is an eigenvector of $W$, corresponding to the second largest eigenvalue, i.e. $W \xx^{(0)} = \lambda_2 \xx^{(0)}$ and we set $y_i$ such that $\yy = \1 + \xx^{(0)}$. With this choice of $\yy$, $\bar\zeta^2 = \norm{\xx^{(0)}}_2^2$. It will be also useful to note that the average $\bar\xx^{(0)} = \0$ since it is orthogonal to $\1$, the eigenvector of $W$ corresponding to the largest eigenvalue. We use the notation $\bar \zz := \frac{1}{n}\1\1^\top \zz$.
	
	We start the proof by decomposing the error $\frac{1}{n}\norm{\xx^{(t)} - \bar\yy}^2_2$ on consensus and optimization terms 
	\begin{align*}
	\frac{1}{n}\norm{\xx^{(t)} - \bar\yy}^2_2 = \frac{1}{n}\norm{\xx^{(t)}  - \bar\xx^{(t)} + \bar\xx^{(t)} - \bar\yy}^2_2 = \frac{1}{n}\norm{\xx^{(t)}  - \bar\xx^{(t)}}_2^2 + \frac{1}{n}\norm{\bar\xx^{(t)} - \bar\yy}_2^2.
	\end{align*}
	Using that for our chosen functions $\nabla f(\xx) = \xx - \yy$, we can estimate the \textbf{optimization term} as 
	\begin{align*}
	\norm{\bar\xx^{(t)} - \bar\yy}_2^2 = \norm{(1 - \eta) \bar\xx^{(t - 1 )} + \eta \bar\yy - \bar\yy}_2^2 = (1 - \eta)^2 \norm{\bar\xx^{(t - 1)} - \bar\yy}_2^2 = (1 - \eta)^{2 t} \norm{\bar\xx^{(0)} - \bar\yy}_2^2 =  (1 - \eta)^{2 t} n. 
	\end{align*}
	
	For the \textbf{consensus term},
	\begin{align*}
	\norm{\xx^{(t)}  - \bar\xx^{(t)}}_2^2 &= \norm{\left(W - \frac{\1\1^\top }{n} \right) \left(\xx^{(t + \frac{1}{2})} - \bar \xx^{(t + \frac{1}{2})} \right)  }_2^2 = \norm{\left(W - \frac{\1\1^\top }{n} \right) \left((1 - \eta)\left(\xx^{(t)} - \bar \xx^{(t)} \right)  + \eta(\yy - \bar \yy)\right) }_2^2 =\\
	&= \norm{\left(W - \frac{\1\1^\top }{n} \right)^t(1 - \eta)^t \xx^{(0)} + \eta \sum_{\tau=0}^{t - 1} (1 - \eta)^\tau \left(W - \frac{\1\1^\top }{n} \right)^{\tau + 1} (\yy - \bar \yy) }_2^2 = \\
	&= \norm{\lambda_2^t(1 - \eta)^t \xx^{(0)} + \eta \sum_{\tau=0}^{t - 1} (1 - \eta)^\tau \lambda_2^{\tau + 1} \xx^{(0)} }_2^2 \\ & = \left( \lambda_2^t(1 - \eta)^t + \eta \sum_{\tau=0}^{t - 1} (1 - \eta)^\tau \lambda_2^{\tau + 1} \right)^2 \norm{\xx^{(0)}}_2^2 \\
	& \geq \left( \lambda_2^{2t}(1 - \eta)^{2t} + \eta^2 \left( \sum_{\tau=0}^{t - 1} (1 - \eta)^\tau \lambda_2^{\tau + 1}\right)^2 \right) n \bar\zeta^2
	\end{align*}

	In order to guarantee error less than $\epsilon$, it is necessary to have simultaneously both optimization and consensus terms less than $\epsilon$, therefore it is required that
	\begin{align}
	(1-\eta)^{2t} &\leq \epsilon\label{eq:first}\\
	(1-\eta)^{2t} \lambda_2^{2t}  &\leq \frac{\epsilon}{\bar\zeta^2} \label{eq:second} \\
	\eta \left( \sum_{\tau=0}^{t - 1} (1 - \eta)^\tau \lambda_2^\tau\right)  =  \eta \frac{1-  (1-\eta)^{t} \lambda_2^{t} }{1-(1-\eta)\lambda_2} & \leq \sqrt{\frac{\epsilon }{\bar\zeta^2\lambda_2^2 }} \label{eq:third}
	\end{align}
	Equations \eqref{eq:second}, \eqref{eq:third} imply
	\begin{align*}
	\eta \leq \sqrt{\frac{\epsilon }{\bar\zeta^2 \lambda_2^2}} \frac{1-(1-\eta)\lambda_2}{1 - \sqrt{\nicefrac{\epsilon }{\bar\zeta^2}}} \leq \sqrt{\frac{\epsilon }{\bar\zeta^2 \lambda_2^2}} \frac{1 - \lambda_2 + \eta}{1-\sqrt{\nicefrac{\epsilon }{\bar\zeta^2 }}}
	\end{align*}
	Note that $\lambda_2 = \sqrt{1 - p}$, where $p$ is from Assumption \ref{a:avg_distrib}. Using that $\sqrt{1 - p} \geq 1 - p$ for $p \in [0, 1]$, 
	\begin{align*}
	\eta \leq \sqrt{\frac{\epsilon }{\bar\zeta^2 (1 - p)}} \frac{1 - \sqrt{1 - p} + \eta}{1-\sqrt{\nicefrac{\epsilon }{\bar\zeta^2}}} \leq \sqrt{\frac{\epsilon }{\bar\zeta^2 (1 - p)}} \frac{p + \eta}{1-\sqrt{\nicefrac{\epsilon }{\bar\zeta^2}}} 
	\end{align*}
	And therefore using that $\sqrt{1 - p} \leq 1$ and for $\epsilon \leq \frac{\bar\zeta^2 (1 - p)}{16}$,
	\begin{align*}
	\eta \leq  \frac{ \sqrt{\nicefrac{\epsilon }{[\bar\zeta^2(1 - p)]}} p }{1 - (\frac{\sqrt{1 - p} + 1}{\sqrt{1 - p}})\sqrt{\nicefrac{\epsilon }{\bar\zeta^2}}} \leq \frac{ \sqrt{\nicefrac{\epsilon }{[\bar\zeta^2(1 - p)]}} p }{1 - 2 \sqrt{\nicefrac{\epsilon }{[\bar\zeta^2(1 - p)]}}} \leq 2 \sqrt{\nicefrac{\epsilon }{[\bar\zeta^2(1 - p)]}} p
	\end{align*}
	
	With this upper bound on $\eta$, the inequality \eqref{eq:first} gives a lower bound on $t$:
	\begin{align}
	t \geq \frac{\log \frac{n}{\epsilon}}{ -2 \log(1-\eta)} \geq \frac{\log \frac{1}{\epsilon}}{2 \eta} \geq \frac{\bar\zeta \sqrt{1 - p} \log \frac{1}{\epsilon}}{4 \sqrt{\epsilon}p},
	\end{align}
	here we used that $\log(1-\eta) \geq -\eta$ for $\eta \leq \frac{4}{5}$. %
\end{proof}

\section{Additional Experiments to Verify the $\cO\bigl(\frac{1}{T^2}\bigr)$ Term}
\label{sec:additionalexp}

In Theorem~\ref{thm:summary} we proved an upper bound and in Theorem~\ref{thm:lower_bound} we proved a lower bound, that indicates that in the noiseless ($\bar\sigma^2 = 0$) strongly convex case the convergence is not linear when $\bar \zeta^2 > 0$. In this section we verify numerically that this rate indeed reflects tightly the convergence behavior of decentralized SGD.

We consider the same setting as in Section~\ref{sec:experiments} before, with $\bar\sigma^2 = 0$, $\bar\zeta^2 =10$, $n=25$, and $d=10$.

For both ring and 2-$d$ torus (grid), we vary the target accuracy ($\epsilon$) and tune the stepsize to find the smallest number of iterations required ($T_{\epsilon}$) to achieve this target accuracy. In Figure \ref{fig:T2exp} we depict the results, where x-axis is $\frac{1}{\sqrt{\epsilon}}$ and y-axis is $T_{\epsilon}$. Based on the Theorem \ref{thm:summary} for strongly convex case, ideally each of them should be a line, as we observe in the plots. Moreover, the ratio of the slopes of these lines is $30.2/2.3 = 13.13$ which matches the ratio of the spectral gap of these graphs ($p_{\rm grid}/p_{\rm ring} = 0.276/0.021 = 13.142$), as it is shown in Theorems \ref{thm:summary} and \ref{thm:lower_bound}.

\begin{figure*}[t]
	\centering
	\begin{minipage}{\textwidth}
		\centering
		\hfill
		\includegraphics[width=0.9\linewidth]{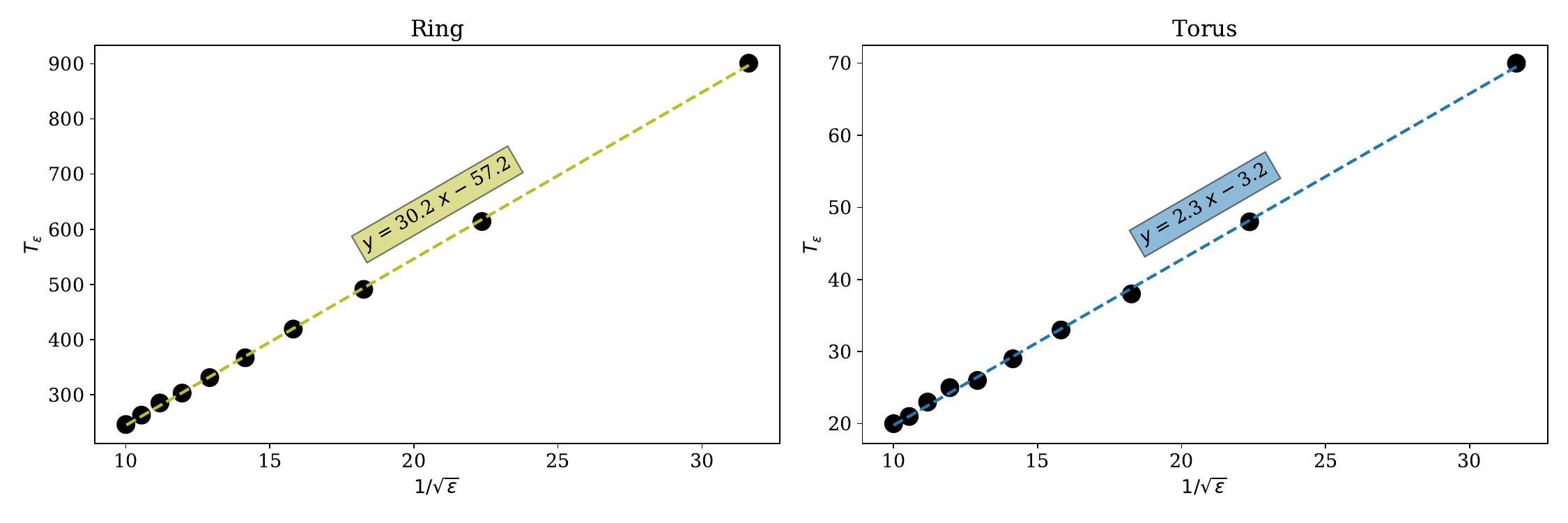}
		\hfill \null
		\caption{%
		Verifying the $\cO\bigl(\frac{\bar \zeta^2}{p^2 T^2}\bigr)$ convergence for the strongly convex noiseless ($\hat \sigma^2=0$) case. Number of iterations to converge to target accuracy $\epsilon$ on ring (left) and 2-$d$ torus (right).
		}%
		\label{fig:T2exp}
	\end{minipage}%
\end{figure*}

\end{document}